%% file: main.tex
\documentclass[11pt,oneside]{jmlr}
\makeatletter
\def\ps@jmlrtps{%
  \let\@mkboth\@gobbletwo
  \def\@oddhead{}\let\@evenhead\@oddhead
  \def\@oddfoot{\hfill\small\rmfamily\thepage\hfill}\let\@evenfoot\@oddfoot
}
\def\ps@jmlrps{%
  \let\@mkboth\@gobbletwo
  \def\@oddhead{}\let\@evenhead\@oddhead
  \def\@oddfoot{\hfill\small\rmfamily\thepage\hfill}\let\@evenfoot\@oddfoot
}
\makeatother
\renewenvironment{proof}[1][\proofname]{\par\noindent{\bfseries\upshape #1\ }}{\jmlrQED}
\firstpageno{1}
\hypersetup{colorlinks=true,linkcolor=black,citecolor=black,urlcolor=black}

\input{math_commands.tex}

\usepackage{mathtools}
\usepackage{booktabs,array,multirow}
\usepackage{wrapfig}
\usepackage{tikz-cd}
\usepackage{enumitem}
\usepackage{microtype}
\SetAlgoVlined
\numberwithin{equation}{section}
\numberwithin{theorem}{section}

\newcommand{\one}{\mathbf 1}
\newcommand{\zero}{\mathbf 0}
\newcommand{\K}{\mathcal K}

\newcommand{\ip}[2]{\langle #1,#2\rangle}
\newcommand{\norm}[1]{\lVert #1\rVert}
\newcommand{\Reg}{\operatorname{Reg}}
\newcommand{\OPT}{\operatorname{OPT}}

\newcommand{\diam}{\operatorname{diam}}

\title{Online Non-Monotone DR-Submodular Maximization\\Beyond $1/e$ with Feasible Gradient Queries}
\author{%
  \Name{Vaneet Aggarwal} \Email{vaneet@purdue.edu}\\
  \Name{Yiyang Lu} \Email{yiyanglu@purdue.edu}\\
  \addr Purdue University, West Lafayette, IN 47906, USA
}
\hypersetup{
  pdftitle={Online Non-Monotone DR-Submodular Maximization Beyond 1/e with Feasible Gradient Queries},
  pdfauthor={Vaneet Aggarwal, Yiyang Lu}
}

\begin{document}
\maketitle

\begin{abstract}
We study adversarial online maximization of nonnegative, non-monotone DR-submodular functions over compact convex down-closed sets $\K\subseteq[0,1]^d$ when gradients may be queried only at feasible points. The online approximation coefficients $4/9$ \citep{aggarwal2026geometry} and $2\sqrt3-3\approx0.4641$ \citep{aggarwal2026du} rely on queries along chains that start from the all-ones vector and leave $\K$. With every query in $\K$, the best known coefficient was $1/e$. We give an online algorithm with coefficient $0.4013$ that, in each round, plays a random action, queries one gradient at a point of $\K$, and makes one projection. Its expected approximate regret is $O(RB(1+R)\sqrt T)$, where $R\le\sqrt d$ bounds the norm of points of $\K$ and $B$ bounds the gradient norm on $\K$. For a budget $\sum_ix_i\le k$, the bound is $O(kB\sqrt T)$. The algorithm keeps a state $x\in\K$, and the analysis applies the coordinate-chain comparison of \citet{aggarwal2026geometry} inside the box $[0,x]$, which keeps every query feasible. A second comparison, along a path from $0$ in the direction $x$, cancels the extra terms that the box creates. Because the queries are feasible, they can be played, which gives regret $O(T^{2/3})$ under semi-bandit feedback and $O(T^{4/5})$ under bandit feedback with bounded values. Unbiased stochastic gradients with bounded second moment keep the $\sqrt T$ rate, and on a fixed function the algorithm attains $0.4013\,\OPT-\varepsilon$ in expectation with $O(1+R^2(1+R)^2B^2/\varepsilon^2)$ feasible gradient queries.
\end{abstract}

\section{Introduction}\label{sec:introduction}
Diminishing-returns (DR) submodular functions model resource allocation, influence maximization and mean-field inference \citep{bian2017structure,bian2019optimal}. A differentiable function on $[0,1]^d$ is DR-submodular when each partial derivative can only decrease as the argument increases, and it is non-monotone when additional allocation may reduce the reward. A feasible set is down-closed when lowering coordinates of a feasible allocation keeps it feasible. In the online problem the objective changes from round to round, and the learner must choose an allocation before seeing the function that evaluates it \citep{chen2018online}.

Offline, \citet{buchbinder2024bounds} obtain approximation coefficient $0.401$ under smoothness and constraint-oracle assumptions. Online, over down-closed convex sets, coefficient $1/e$ is attained by \citet{thang2021online,zhang2023online,pedramfar2024unified,lu2026upper}, and \citet{lu2026upper} attain it with one feasible gradient query per round and $O(\sqrt T)$ approximate regret. Larger online coefficients are known when the gradient may be queried anywhere in the cube $[0,1]^d$: with one gradient query per round and $O(\sqrt T)$ regret, \citet{aggarwal2026geometry} attain $4/9$ and \citet{aggarwal2026du} attain $2\sqrt3-3\approx0.4641$. We study the problem in which every query must lie in the feasible set $\K$, and we raise the coefficient in this setting from $1/e$ to $0.4013$, keeping one gradient query per round and the $\sqrt T$ dependence on the horizon:
\begin{equation}\label{eq:introresult}
\E\sum_{t=1}^T f_t(p_t)\ \ge\ 0.4013\max_{o\in\K}\sum_{t=1}^T f_t(o)-O\bigl(RB(1+R)\sqrt T\bigr),
\end{equation}
where $p_t$ is the action played in round $t$, $R\le\sqrt d$ bounds the Euclidean norm of points of $\K$, and $B$ bounds the gradient norm on $\K$. Compared with \citet{lu2026upper}, the regret term has an extra factor $O(1+R)$, which is $O(\sqrt d)$ in general and $O(\sqrt k)$ for a budget $\sum_ix_i\le k$. Our coefficient is below $4/9$, and whether $4/9$ can be attained with feasible queries is open.

\paragraph{Why feasible queries matter.}
In many applications a gradient query evaluates the system at an allocation, by simulating, deploying or experimenting with it, and returns the sensitivity of the reward to each resource at that allocation. An allocation that violates the constraints cannot be evaluated. Consider a budget $\sum_ix_i\le k$ on $d$ resources with $k<d$, the continuous form of a cardinality constraint: a query is a spending plan, and a plan that spends more than $k$ cannot be run. Our algorithm queries gradients only at points of $\K$, so every query is itself a feasible allocation, and the gradient bound $B$ is needed only on $\K$. A feasible query point can also be played as an action. This is what the conversions of \citet{pedramfar2024linear} to semi-bandit and bandit feedback need, and it lets them apply to our algorithm (Section~\ref{sec:regret}).

The cube-query algorithms above do not have this property. The comparison behind the coefficient $4/9$ \citep[Lemma~4.1]{aggarwal2026geometry} averages partial derivatives along two coordinate chains that meet at the action: a lower chain that starts at $\zero$ and raises one coordinate at a time, and an upper chain that starts at $\one$ and lowers one coordinate at a time. The one-query estimator of \citet{aggarwal2026geometry} samples a point on one of the two chains, and on the upper chain every coordinate not yet processed equals $1$. Under the budget above, every upper-chain point with more than $k$ coordinates equal to $1$ violates the constraint, starting with $\one$ itself. The regret bound of \citet{aggarwal2026geometry} also uses a gradient bound on the whole cube. The coefficient $2\sqrt3-3$ of \citet{aggarwal2026du} also uses chains that reach $\one$. Proposition~\ref{prop:box} applies the same chain comparison inside the box $[0,x]=\{z:0\le z\le x\}$ below the state. Both chains, and hence all queries, then stay in $\K$.

\paragraph{Why online is harder than offline.}
An offline solver may query one objective as often as it needs before returning a point. An online learner earns the reward of its action before it receives any information about the current function, and the next function may be chosen adversarially. Running an online algorithm on a constant sequence and returning a uniformly sampled action gives an offline algorithm, but an offline guarantee gives no regret bound for a changing sequence. Our guarantee holds for every sequence fixed in advance, with no stationarity or independence assumption on the functions.

\paragraph{What the analysis needs.}
The algorithm keeps a point $x\in\K$, its state, computed from past gradients. At every feasible state $x$, we need $\alpha f(o)-\E[f(p(x))]\le\ip{G(f,x)}{o-x}$ for every $o\in\K$, with an action distribution fixed before observing $f$ and a vector $G(f,x)$ that does not depend on $o$ and that one feasible gradient query estimates without bias. Such an inequality reduces the problem to online linear optimization, which projected gradient ascent solves \citep{pedramfar2024linear,lu2026upper}. The comparisons of \citet{aggarwal2026geometry,aggarwal2026du} have this form, but estimating their vectors requires queries outside $\K$. The offline algorithm of \citet{buchbinder2024bounds} does not directly provide such an inequality, because each of its steps uses the objective before a point is chosen. That algorithm first computes an approximate local maximum $z$ of the objective by a Frank--Wolfe procedure. It then runs a box maximization and a guided measured continuous greedy on the same objective, recursively, with the outputs of earlier calls as new guiding points. It also enumerates guesses of three objective values that involve the unknown optimum. Online, these objective-dependent computations cannot determine the current action, which must be committed before feedback on $f_t$ is received, and the relevant $o$ maximizes the sum of all functions rather than the current one. Theorem~\ref{thm:comparison} provides the required inequality with $\alpha_*>0.4013$ and a vector $G(f,x)$ built from gradients in $[0,x]$, and the online reduction converts it into an approximate-regret bound with the same coefficient.

\begin{wrapfigure}{r}{0.35\textwidth}
\vspace{-.2in}
\centering
\[
\begin{gathered}
\boxed{\text{Proposition~\ref{prop:trajectory}}}\\
+\\
\boxed{\text{Proposition~\ref{prop:box}}}
\\[3pt]
\Downarrow
\makebox[0pt][l]{\text{\small cancellation}}
\\[3pt]
\boxed{\text{Theorem~\ref{thm:comparison}}}
\\[3pt]
\makebox[0pt][r]{\text{\small one-query}}
\Downarrow
\makebox[0pt][l]{\text{\small sampling}}
\\[3pt]
\boxed{\text{Algorithm~\ref{alg:main}}}
\\[3pt]
\Downarrow
\\[3pt]
\boxed{\text{Theorem~\ref{thm:regret}: }O(\sqrt T)\text{ regret}}
\end{gathered}
\]
\vspace{-.2in}
\caption[Organization of results.]{\raggedright Organization of results.}
\label{fig:organization}
\vspace{-.2in}
\end{wrapfigure}

\paragraph{How the coefficient $0.4013$ is obtained.}
We apply two lower-bound arguments at the same state $x$. The first follows a continuous path from $0$ in the direction of $x$, started with a delay, as in \citet{buchbinder2024bounds}, and its queries lie in $[0,x]$. Proposition~\ref{prop:trajectory} strengthens this comparison by keeping the order-two term of the exponential expansion that the earlier analysis discards. Without this term, the cancellation described below would fail at our parameters. The second is the coordinate-chain comparison of \citet{aggarwal2026geometry}, stated as Lemma~\ref{lem:chain}. Applied to $f$ at $x$, it compares with $f(o)$ directly and gives $4/9$, but its upper chain leaves $\K$. Proposition~\ref{prop:box} applies it instead to $a\mapsto f(x\odot(\one-a))$, where $\odot$ is the coordinatewise product. This substitution maps every chain point into $[0,x]$, so the queries become feasible, the linear term has the same direction $o-x$ as in the first argument, and the endpoint value becomes $f(x)$. The comparator value, however, becomes $f(x\odot o)$ instead of $f(o)$. Both arguments therefore leave terms involving $f(x\odot o)$ and $f(x)$. Theorem~\ref{thm:comparison} matches their coefficients and cancels them through a fixed random mixture of the two actions. What remains is an inner product with one vector, which one feasible gradient query estimates without bias, and the coefficient is $0.4013063665\ldots$.

\paragraph{Contributions.} Figure~\ref{fig:organization} shows how the results depend on each other.

\noindent 1. Theorem~\ref{thm:comparison} proves the inequality described above with coefficient $\alpha_*>0.4013$ and a vector $G(f,x)$ whose gradients are all taken in $[0,x]\subseteq\K$. Together with our one-query implementation, this yields, to our knowledge, the first online approximation guarantee above $1/e$ in which every query lies in $\K$.

\noindent 2. Algorithm~\ref{alg:main} plays a fixed random mixture of two explicit actions, the endpoint of a delayed path and the point $x/(2-x)$, and updates $x$ by projected gradient ascent, with one feasible gradient query, one projection and no local search on the current objective per round. With exact gradients, Theorem~\ref{thm:regret} gives expected approximate regret $O(RB(1+R)\sqrt T)$ against every sequence fixed in advance. This regret bound is $O(DB\sqrt{dT})$ in general and $O(kB\sqrt T)$ for a budget $\sum_ix_i\le k$ with $k\ge1$. The gradient bound is needed only on $\K$, and $\K$ may have empty interior.

\noindent 3. The coefficient $\alpha_*$ carries over to weaker feedback, with the following approximate-regret rates.
\begin{itemize}
\item Unbiased stochastic gradients with bounded conditional second moment, one query per round: $O(\sqrt T)$ (Corollary~\ref{cor:stochastic}).
\item Semi-bandit feedback, where the gradient is observed only at the played point: $O(T^{2/3})$ (Corollary~\ref{cor:semibandit}).
\item One noisy, bounded value query per round: $O(T^{3/4})$ (Corollary~\ref{cor:zeroth}).
\item Bandit feedback with bounded values, where only the value at the played point is observed: $O(T^{4/5})$ (Corollary~\ref{cor:zeroth}).
\end{itemize}
The last three rates use the conversions SFTT and FOTZO of \citet{pedramfar2024linear}, which play query points as actions or perturb them inside $\K$, so they need feasible queries.

\noindent 4. Run on a fixed function with exact gradients, the algorithm reaches expected value at least $\alpha_*$ times the optimum minus $\varepsilon$ with $O(1+R^2(1+R)^2B^2/\varepsilon^2)$ gradient queries and projections, without assuming Lipschitz-continuous gradients (Corollary~\ref{cor:offline}).

\begin{table}[!ht]
\caption{Approximation coefficients, grouped by setting and by where the gradient oracle may be queried. Online errors bound the expected approximate regret with exact gradients, for methods that use one gradient query and one projection per round. The row of \citet{lu2026upper} combines their comparison with projected ascent (Section~\ref{sec:related}). Offline errors are additive losses in expected value. In the rows with queries in $[0,1]^d$, $B$ bounds the gradient on the whole cube. In the other rows it bounds the gradient on $\K$. Our offline row holds for every $\delta\in(0,1)$ with $N_\delta=\lceil\delta^{-2}\rceil$ queries and projections, and $\Gamma_R=O(1+R)$ with $R\le\sqrt d$. The row of \citet{buchbinder2024bounds} holds for every $\delta\in(0,1)$, in $\operatorname{poly}(d,\delta^{-1})$ time, for $L$-smooth objectives.}
\label{tab:comparison}
\centering
\small
\begin{tabular*}{\textwidth}{@{\extracolsep{\fill}}llll@{}}
\toprule
Setting and queries & Reference & Coefficient & Error bound\\
\midrule
\multirow{2}{*}{Online, queries in $\K$}
& \citet{lu2026upper} & $1/e$ & $O(DB\sqrt T)$\\
& This work & $\alpha_*>\mathbf{0.4013}$ & $O(DB\Gamma_R\sqrt T)$\\
\midrule
\multirow{2}{*}{Online, queries in $[0,1]^d$}
& \citet{aggarwal2026geometry} & $4/9$ & $O(DB\sqrt{dT})$\\
& \citet{aggarwal2026du} & $2\sqrt3-3\approx0.4641$ & $O(DB\sqrt{dT})$\\
\midrule
\multirow{2}{*}{Offline, queries in $\K$}
& \citet{buchbinder2024bounds} & $0.4010$ & $O(\delta D^2L)$\\
& This work & $\alpha_*>0.4013$ & $\delta DB\Gamma_R$\\
\midrule
\multirow{2}{*}{Offline, queries in $[0,1]^d$}
& \citet{aggarwal2026geometry} & $4/9$ & $\varepsilon\,\OPT$\\
& \citet{aggarwal2026du} & $2\sqrt3-3$ & $\varepsilon\,\OPT$\\
\bottomrule
\end{tabular*}
\end{table}

\section{Related work}\label{sec:related}

\paragraph{Offline approximation.}
Measured continuous greedy and its successors gave approximation guarantees for constrained non-monotone submodular and DR-submodular maximization \citep{feldman2011unified,bian2017structure,ene2016beyond,buchbinder2019nonsymmetric}. \citet{pedramfar2023unified} analyze Frank--Wolfe type algorithms for monotone and non-monotone DR-submodular maximization over general convex sets in one framework, with deterministic or stochastic access to gradients or function values. \citet[Theorem 1.1]{buchbinder2024bounds} prove $0.401\,\OPT-O(\delta D^2L)$ for every $\delta\in(0,1)$, in $\operatorname{poly}(d,\delta^{-1})$ time, for nonnegative $L$-smooth DR-submodular objectives over meta-solvable down-closed convex bodies. These are bodies over which linear functions can be optimized efficiently after intersecting with polynomially many explicit halfspaces. Our trajectory comparison uses their delayed path with its direction fixed to the state, their exponential expansion (Lemma~4.1) and its two-block case (Lemma~4.2). Their construction chooses guiding points and auxiliary solutions using the current objective, whereas our direction is the state $x$, which past feedback determines, and both our actions are explicit functions of $x$. Corollary~\ref{cor:offline} uses gradient and projection oracles, needs no Lipschitz constant for the gradient, and reaches additive error $\varepsilon$ with $O(1+R^2(1+R)^2B^2/\varepsilon^2)$ queries, all in $\K$. The coefficient $\alpha_*$ of Corollary~\ref{cor:offline} is slightly larger than $0.401$.

\paragraph{Online approximation and query feasibility.}
\citet{thang2021online,zhang2023online,pedramfar2024unified} study online non-monotone DR-submodular maximization over down-closed sets. \citet{lu2026upper} bound the approximation gap at the state, with coefficient $1/e$, by an inner product with a vector that does not depend on the comparator. They estimate this vector from a single gradient and derive online regret bounds. Their query $1-e^{-vx}$, $v\in[0,1]$, satisfies $0\le1-e^{-vx}\le x$, so it is feasible on a down-closed set. \citet{aggarwal2026geometry} prove a comparison of the same form with coefficient $4/9$, and with a larger coefficient when a known multiple of $\one$ lies in $\K$. They estimate its vector from one gradient query in the cube per round. \citet{aggarwal2026du} raise the coefficient to $2\sqrt3-3$ in the same query model. Both vectors average partial derivatives along coordinate chains that reach $\one$, so their queries generally lie outside $\K$. \citet{aggarwal2026geometry} further show that no offline algorithm with polynomially many exact value and gradient queries exceeds $\beta_*\approx0.4704$, and that the same bound limits online learners with polynomial query budgets whose regret is at most a fixed polynomial in the problem parameters times $T^c$, for a fixed $c<1$. The bound $\beta_*$ therefore also applies to algorithms whose queries all lie in $\K$.

The online row for \citet{lu2026upper} in Table~\ref{tab:comparison} uses their estimator \citep[Theorem 1 and Lemma 2]{lu2026upper} scaled by $c_0=1-e^{-1}$, whose norm is at most $c_0B$. Projected gradient ascent then gives regret at most $Dc_0B\sqrt T$. The framework of \citet{lu2026upper} also allows other online linear learners, such as the separation-oracle method in that paper. The table uses projected ascent in both feasible-query online rows so that the two bounds are comparable.

\paragraph{Coordinate chains and the online reduction.}
Changing coordinates one at a time is the template of continuous double greedy \citep{niazadeh2020optimal} and of online unconstrained submodular maximization \citep{roughgarden2018optimal}. \citet[Lemma~4.1]{aggarwal2026geometry} apply it at a prescribed point, order the coordinates by their weights, and keep the endpoint value. Lemma~\ref{lem:chain} is their case of diagonal level zero, and Proposition~\ref{prop:box} applies it inside the box $[0,x]$, so that every query lies in $\K$. We use the reduction of such a comparison to online linear optimization, called OMBQ, the conversion SFTT from one query per round to semi-bandit feedback, the conversion FOTZO from gradient to value queries, and the online-to-batch conversion OTB as formulated by \citet{pedramfar2024linear}. Numbered results of that paper are cited from its revised arXiv version, in which FOTZO runs the base algorithm on the original feasible set. \citet{lu2026upper} apply OMBQ and OTB to non-monotone DR-submodular maximization over down-closed sets. \citet{pedramfar2026weakly} extend the upper-linearizable framework to $\gamma$-weakly $\theta$-up-concave functions, a class of nonnegative monotone objectives that includes monotone DR-submodular and monotone one-sided smooth functions, and \citet{pedramfar2025uniform} give uniform wrappers, which convert online algorithms for concave objectives into algorithms for quadratizable ones. As the linear learner we take projected gradient ascent \citep{zinkevich2003online}, and we use the reduction in the form of \citet[Section~5.2]{aggarwal2026geometry}, stated as Lemma~\ref{lem:reduction}, which allows a random action and estimates with a bounded conditional second moment.

\section{Problem and feedback model}\label{sec:setup}
Let $\K\subseteq[0,1]^d$ be nonempty, compact, convex, and down-closed, so that $\zero\in\K$, and let $D=\diam(\K)\le\sqrt d$. Let $R>0$ satisfy $\norm{x}_2\le R$ for every $x\in\K$. Every point of the cube has norm at most $\sqrt d$, so we may take $R\le\sqrt d$. Since $\K$ lies in the nonnegative orthant, every $x,y\in\K$ satisfy $\norm{x-y}_2^2=\norm{x}_2^2+\norm{y}_2^2-2\ip{x}{y}\le2R^2$, so $D\le\sqrt2R$. Algorithm~\ref{alg:main} takes $D$, $R$ and $B$ as inputs. When $\diam(\K)$ is not known, every result below holds with $D$ replaced by $\sqrt2R$, which also satisfies the bound $D\le\sqrt2R$ used in the simplifications. For the budget set $\{x\in[0,1]^d:\sum_ix_i\le k\}$, $\norm{x}_2^2\le\sum_ix_i\le k$, so $R=\sqrt k$ is valid. An exact Euclidean projection oracle $\Pi_\K$ is available, and projections are counted separately from gradient queries. Products $x\odot o$, quotients, exponentials, and vector inequalities are coordinatewise. We write $x\oplus o=x+o-x\odot o$, $[0,x]=\{z:0\le z\le x\}$, $[d]=\{1,\ldots,d\}$ and $(a)_+=\max\{a,0\}$.

Each reward $f_t:[0,1]^d\to\R_{\ge0}$ is continuously differentiable and DR-submodular:
\begin{equation}\label{eq:assumptions}
u\le v\ \Longrightarrow\ \nabla f_t(u)\ge\nabla f_t(v),
\qquad \sup_{z\in\K}\norm{\nabla f_t(z)}_2\le B.
\end{equation}
Nonnegativity and DR-submodularity are assumed on the whole cube $[0,1]^d$, and gradients on its boundary are defined by continuous extension. The gradient bound $B$ is needed only on $\K$, no Lipschitz condition on the gradient is assumed, and $\K$ need not contain a point with all coordinates positive. Sets with empty interior, such as $\{x\in[0,1]^d:\sum_i x_i\le1,\ x_d=0\}$, are therefore covered. The algorithm works only with boxes $[0,x]$ for $x\in\K$, which lie in $\K$, and needs no neighborhood of a point.

An oblivious adversary fixes the sequence $f_1,\ldots,f_T$ in advance and does not reveal it. In round $t$ the learner commits to an action $p_t\in\K$, earns the reward $f_t(p_t)$, and then queries the full gradient of $f_t$ at one point $z_t\in\K$ of its choice, which may differ from $p_t$. \citet{pedramfar2024linear} call this full-information feedback with one query per round. Under semi-bandit feedback, the learner may query the gradient only at the point it plays, that is, $z_t=p_t$. Under zeroth-order feedback, queries return function values instead of gradients, and under bandit feedback the learner observes only the value $f_t(p_t)$. Algorithm~\ref{alg:main} uses only the queried gradient and no function values. The $\alpha$-approximate regret is
\begin{equation}\label{eq:regret}
\Reg_\alpha(T)=\alpha\max_{o\in\K}\sum_{t=1}^T f_t(o)-\E\sum_{t=1}^T f_t(p_t),
\end{equation}
where the expectation is over the learner's random draws and, when present, the randomness of the feedback oracle. The point $o$ that attains the maximum is called the comparator. The algorithm uses $D$, $R$, $B$ and $T$ to set its step size and branch probabilities. If $D=0$, then $\K=\{0\}$. If $B=0$, then $\nabla f_t=0$ on the convex set $\K$, so $f_t$ is constant on $\K$. In both cases the learner plays and queries $0$, and $\Reg_\alpha(T)=(\alpha-1)\E\sum_tf_t(0)\le0$ for every $\alpha\le1$. From now on we assume $D,B>0$.

The following lattice inequality \citep{bian2017structure} is used repeatedly:
\begin{equation}\label{eq:lattice}
f(x\odot o)+f(x\oplus o)\le2f(x)+\ip{\nabla f(x)}{o-x}.
\end{equation}
To prove \eqref{eq:lattice}, recall that a DR-submodular function is concave along every direction $u\ge0$ and every direction $u\le0$ \citep{bian2017structure}. The direction $x\odot o-x=-x\odot(\one-o)$ is nonpositive, so concavity gives $f(x\odot o)\le f(x)+\ip{\nabla f(x)}{x\odot o-x}$. The direction $x\oplus o-x=(\one-x)\odot o$ is nonnegative, so concavity gives $f(x\oplus o)\le f(x)+\ip{\nabla f(x)}{x\oplus o-x}$. Adding the two inequalities and using $(x\odot o-x)+(x\oplus o-x)=o-x$ gives \eqref{eq:lattice}. The point $x\oplus o$ appears only in the analysis and need not lie in $\K$.

\section{A feasible first-order comparison with coefficient 0.4013}\label{sec:comparison}
This section proves an inequality that holds at every state $x\in\K$ and simultaneously for every comparator $o\in\K$:
\[
\alpha_* f(o)-\E[f(p(x))\mid f,x]\le\ip{G(f,x)}{o-x}.
\]
Here $p(x)$ is a random action whose distribution depends only on $x$, and the vector $G(f,x)$ does not depend on $o$. Because the state and the action distribution are fixed before any feedback on $f$ is received, the inequality can be applied in every round of the online algorithm.

\subsection{A strengthened delayed-trajectory comparison}\label{sec:trajectory}
For a state $x\in\K$, a delay $s\in(0,1)$ and a time $v\in[0,1]$, define
\begin{align}
Y_v(x)&=(\one-x)\odot(\one-e^{-vx})+x\odot(\one-e^{-(v-s)_+x}),\label{eq:path}\\
\omega_v(x)&=(\one-x)\odot e^{-vx}+\mathbf1_{\{v\ge s\}}x\odot e^{-(v-s)x},\label{eq:omega}\\
q_s(f,x)&=\int_0^1e^{v-1}\omega_v(x)\odot\nabla f(Y_v(x))\,dv.\label{eq:q}
\end{align}
This is the delayed path of \citet{buchbinder2024bounds} with its direction fixed to $x$. Since $1-e^{-c}\le c$ for $c\ge0$, $v\le1$ and $(v-s)_+\le1$, each coordinate satisfies $0\le Y_v(x)_i\le(1-x_i)x_i+x_i^2=x_i$. Also $0\le\omega_v(x)\le\one$, and differentiating \eqref{eq:path} gives $\dot Y_v=\omega_v\odot x$ for every $v\ne s$. Hence the action $Y_1(x)$ and every point $Y_v(x)$ at which \eqref{eq:q} evaluates $\nabla f$ lie in $[0,x]\subseteq\K$.

Define the scalar coefficients
\begin{equation}\label{eq:trajectorycoeff}
\theta_s=\frac{(2-s)e^s-1}{e},\qquad
\mu_s=\frac{e^s-1}{e},\qquad
\zeta_s=\frac{(2-s)e^s-2}{e}-\frac{(1-s)^2}{2e}\left(1+\frac{3s}{4}\right).
\end{equation}

The next proposition bounds the value at the path endpoint $Y_1(x)$ from below by $\theta_sf(o)$, minus the values at $x\odot o$ and $x\oplus o$ and a linear term in $o-x$. Its second part uses \eqref{eq:lattice} to replace the value at $x\oplus o$ by $f(x)$ and a gradient term.
\begin{proposition}\label{prop:trajectory}
For every $x,o\in\K$,
\begin{equation}\label{eq:trajectorycompare}
f(Y_1(x))\ge\theta_s f(o)-\mu_s f(x\odot o)-\zeta_s f(x\oplus o)-\ip{q_s(f,x)}{o-x}.
\end{equation}
Consequently, for $\ell\ge\max\{0,\zeta_s\}$,
\begin{equation}\label{eq:trajectoryfree}
\begin{split}
\theta_s f(o)-f(Y_1(x))\le{}&(\mu_s-\ell)f(x\odot o)+2\ell f(x)\\
&+\ip{q_s(f,x)+\ell\nabla f(x)}{o-x}.
\end{split}
\end{equation}
\end{proposition}
Proposition~\ref{prop:trajectory} is proved in Appendix~\ref{app:trajectory-integration}, with auxiliary lemmas in Appendix~\ref{app:trajectory}. Following \citet[Lemmas 4.1--4.2]{buchbinder2024bounds}, the proof expands $f$ along the path and splits the bound into an early phase $v<s$ and a late phase $v\ge s$. Their analysis discards the order-two term of this expansion with one integration time in each phase, and we keep it. The expansion (Lemma~\ref{lem:exponential}) uses only directional concavity and first derivatives, so keeping this term needs no second derivatives of $f$. With $U=x\oplus o$, the term is evaluated at $H=U+(\one-o)\odot x\odot(\one-x)^2$, and Lemma~\ref{lem:late} shows that $f(H)\ge\tfrac34f(U)$, which lowers $\zeta_s$ by $3s(1-s)^2/(8e)$, the last term in \eqref{eq:trajectorycoeff}. Section~\ref{sec:cancellation} takes $\ell=\mu_s/3$, which requires $\ell\ge\zeta_s$. At $s=199/500$ this choice gives $\ell-\zeta_s>0.004823$, while without the order-two term the same difference is below $-0.015$. The comparator $o$ in Proposition~\ref{prop:trajectory} is arbitrary in $\K$, which the online analysis needs because its comparator maximizes the sum of all rewards rather than the current $f$.

\subsection{A feasible box comparison with the same state}\label{sec:box}
With $\ell=\mu_s/3$, the value chosen in Section~\ref{sec:cancellation}, the right side of \eqref{eq:trajectoryfree} contains $\tfrac{2\mu_s}{3}\bigl[f(x\odot o)+f(x)\bigr]$. This section bounds the sum $f(x\odot o)+f(x)$, not $f(x\odot o)$ alone, by the reward of a feasible action plus a linear term in $o-x$. The tool is the coordinate-chain comparison of \citet[Lemma~4.1]{aggarwal2026geometry}. This comparison changes coordinates one at a time between $0$, $\one$ and a point $r$, as in analyses of double-greedy algorithms \citep{niazadeh2020optimal,roughgarden2018optimal}, but it holds at a prescribed point $a$ rather than at an optimizer of $F$, and it keeps the endpoint value $F(0)$. We state its case of diagonal level zero.

Let $F$ be nonnegative, continuously differentiable and DR-submodular on the cube. For $a\in[0,1]^d$, set $r_i=a_i/(1+a_i)$ and $w_i=a_i(1-a_i)$. Let $\sigma$ be the permutation of $[d]$ that orders the coordinates by nondecreasing $w_i$, breaking ties by index. The chains stay in the original coordinates, and $\sigma$ only fixes the order in which coordinates are processed. Let $L^0=0$ and $U^0=\one$, and for $j=1,\ldots,d$ obtain $L^j$ and $U^j$ from $L^{j-1}$ and $U^{j-1}$ by setting coordinate $\sigma_j$ to $r_{\sigma_j}$. For a coordinate $i$ with position $j$, that is, $\sigma_j=i$, define
\begin{equation}\label{eq:chainfield}
\begin{split}
\underline A_i&=\int_0^1\partial_iF(L^{j-1}+v r_i e_i)\,dv,\qquad
\overline A_i=\int_0^1\partial_iF(U^{j-1}-v(1-r_i)e_i)\,dv,\\
g_i(F,a)&=(1-r_i)\underline A_i+r_i\overline A_i.
\end{split}
\end{equation}
The integrals are averages, so the definition also applies when $r_i=0$.

The lemma compares $F(b)$, for every point $b$ of the cube, with the value of $F$ at the point $r$ where the two chains meet.
\begin{lemma}[{\citealp[Lemma~4.1]{aggarwal2026geometry}}]\label{lem:chain}
For every $a,b\in[0,1]^d$,
\begin{equation}\label{eq:chaincompare}
F(b)+F(0)-\tfrac94 F(r)\le\ip{g(F,a)}{b-a}.
\end{equation}
\end{lemma}
At diagonal level zero, the action map of \citet{aggarwal2026geometry} is $a\mapsto r$, their coordinate weights are the $w_i$ with the same tie rule, their vector field is $g(F,a)$, and their coefficient is $4/9$, so \eqref{eq:chaincompare} is their inequality~(16). Their proof uses $F\ge0$ at the upper-chain points $U^i$.

Applied to $F=f$ with $a=x$ and $b=o$, Lemma~\ref{lem:chain} gives $\tfrac49f(o)-f(x/(1+x))\le\tfrac49\ip{g(f,x)}{o-x}$, the comparison behind the online coefficient $4/9$ of \citet{aggarwal2026geometry}. The vector $g(f,x)$ averages $\partial_if$ along the upper chain, which starts at $\one$ and keeps the value $1$ in every coordinate not yet processed, so these query points generally lie outside $\K$.

We apply Lemma~\ref{lem:chain} to the function $J_x$ below. With this substitution the endpoint value $J_x(0)$ equals $f(x)$, and the linear term is in the direction $o-x$, as in \eqref{eq:trajectoryfree}.

\begin{proposition}\label{prop:box}
For $x,o\in\K$, define
\begin{equation}\label{eq:boxdefs}
J_x(a)=f(x\odot(\one-a)),\qquad P(x)=\frac{x}{2-x},\qquad
h(f,x)=-g(J_x,\one-x).
\end{equation}
Then
\begin{equation}\label{eq:boxcompare}
f(x\odot o)+f(x)-\tfrac94 f(P(x))\le\ip{h(f,x)}{o-x}.
\end{equation}
Both $P(x)$ and all original gradient locations defining $h(f,x)$ belong to $[0,x]\subseteq\K$.
\end{proposition}
\begin{proof}
Fix $x\in\K$. The map $a\mapsto x\odot(\one-a)$ sends $[0,1]^d$ into $[0,x]$, so $J_x$ is defined and nonnegative on the cube. By the chain rule, $\nabla J_x(a)=-x\odot\nabla f(x\odot(\one-a))$. If $a\le a'$, then $x\odot(\one-a)\ge x\odot(\one-a')$. DR-submodularity of $f$ gives $\nabla f(x\odot(\one-a))\le\nabla f(x\odot(\one-a'))$, and multiplying by $-x\le0$ gives $\nabla J_x(a)\ge\nabla J_x(a')$. Hence $J_x$ is DR-submodular, and Lemma~\ref{lem:chain} applies to $F=J_x$. Take $a=\one-x$ and $b=\one-o$. Then $r_i=(1-x_i)/(2-x_i)$ and $1-r_i=1/(2-x_i)$, so $x\odot(\one-r)=x/(2-x)=P(x)$ and
\[
J_x(b)=f(x\odot o),\qquad J_x(0)=f(x),\qquad J_x(r)=f(P(x)).
\]
Moreover $b-a=x-o$, so $\ip{g(J_x,\one-x)}{b-a}=\ip{-g(J_x,\one-x)}{o-x}=\ip{h(f,x)}{o-x}$. Substituting these values into \eqref{eq:chaincompare} gives \eqref{eq:boxcompare}. For feasibility, $2-x_i\in[1,2]$, so $0\le P(x)\le x$. The gradients that define $h(f,x)$ are taken at the points $x\odot(\one-z)$ with $z$ on the chains of Lemma~\ref{lem:chain}. Each such point lies in $[0,x]$, which is contained in $\K$ because $\K$ is down-closed.
\end{proof}
The point $a=\one-x$ need not lie in $\K$, because only the points $x\odot(\one-z)\in[0,x]$ are queried.

\subsection{Cancellation and the 0.4013 first-order inequality}\label{sec:cancellation}
Both bounds now have a linear term in the direction $o-x$. To cancel their nonlinear terms, first match the coefficients of $f(x\odot o)$ and $f(x)$ in \eqref{eq:trajectoryfree}:
\[
\mu_s-\ell=2\ell
\quad\Longrightarrow\quad
\ell=\frac{\mu_s}{3},\qquad b_s=\frac{2\mu_s}{3}.
\]
Multiplying \eqref{eq:boxcompare} by $b_s$ makes the coefficient of $f(P(x))$ equal to $\tfrac94b_s$. The two actions are then played with probabilities proportional to $1$ and $\tfrac94b_s$. Writing $\lambda$ for the probability of $Y_1(x)$, this requires
\[
\frac94\lambda b_s=1-\lambda
\quad\Longrightarrow\quad
\lambda=\frac{1}{1+3\mu_s/2}.
\]
The delay $s$ is the only free parameter. With these choices the coefficient is $\alpha(s)=\lambda\theta_s=N(s)/M(s)$, where $N(s)=(2-s)e^s-1$ and $M(s)=e+\tfrac32(e^s-1)$. Since $N'(s)=(1-s)e^s$ and $M'(s)=\tfrac32e^s$,
\[
N'M-NM'=e^s\bigl[(1-s)(e-\tfrac32)+\tfrac32(1-s)e^s-\tfrac32(2-s)e^s+\tfrac32\bigr]=e^s\psi(s),
\]
where $\psi(s)=(e-\tfrac32)(1-s)+\tfrac32-\tfrac32e^s$. The function $\psi$ is decreasing, with $\psi(0)=e-\tfrac32>0$ and $\psi(1)=\tfrac32(1-e)<0$, so $\alpha$ increases up to the unique zero $s^\star\approx0.39804$ of $\psi$ and decreases after it. Thus $s^\star$ is the best delay within this family of matched mixtures, and we fix the nearby rational value
\begin{equation}\label{eq:parameters}
s=\frac{199}{500},\qquad \ell=\frac{\mu_s}{3},\qquad b_s=\frac{2\mu_s}{3},\qquad
\lambda=\frac{1}{1+3\mu_s/2},\qquad \alpha_* =\lambda\theta_s.
\end{equation}
To bound the resulting constants, let $E_n(c)=\sum_{j=0}^nc^j/j!$ for $0\le c\le1$. The tail $e^c-E_n(c)=\sum_{j>n}c^j/j!$ has nonnegative terms, its first term is $c^{n+1}/(n+1)!$, and the ratio of consecutive terms is $c/(j+1)\le c/(n+2)$ for $j\ge n+1$. Comparing with a geometric series gives
\begin{equation}\label{eq:taylorbound}
E_n(c)\le e^c\le E_n(c)+\frac{c^{n+1}}{(n+1)!}\cdot\frac{1}{1-c/(n+2)}.
\end{equation}
With $n=18$, applied at $c=1$ and $c=s$, \eqref{eq:taylorbound} encloses $e$ and $e^s$ between explicit rational numbers. By \eqref{eq:trajectorycoeff} and \eqref{eq:parameters},
\[
\alpha_*=\frac{(2-s)e^s-1}{e+\tfrac32(e^s-1)},
\qquad
\ell-\zeta_s=\frac{(s-\tfrac53)e^s+\tfrac53+\tfrac12(1-s)^2(1+\tfrac{3s}{4})}{e}.
\]
The numerator of $\alpha_*$ is positive and increasing in $e^s$, and its denominator is positive and increasing in $e$ and $e^s$. Substituting the lower bound on $e^s$ in the numerator and the upper bounds on $e$ and $e^s$ in the denominator therefore gives a lower bound on $\alpha_*$. In the numerator of $\ell-\zeta_s$ the coefficient $s-\tfrac53$ of $e^s$ is negative, so its lower bound uses the upper bound on $e^s$. This lower bound is positive, and dividing it by the upper bound on $e$ gives a lower bound on $\ell-\zeta_s$. Evaluating these rational numbers gives
\begin{equation}\label{eq:constants}
\ell-\zeta_s>0.004823,\qquad \alpha_*>0.40130636>0.4013.
\end{equation}
In particular $\ell\ge\max\{0,\zeta_s\}$, so \eqref{eq:trajectoryfree} holds with this $\ell$, and $\mu_s-\ell=2\ell=b_s$.

With these parameters, Propositions~\ref{prop:trajectory} and~\ref{prop:box} combine into the inequality stated at the start of this section.
\begin{theorem}[Feasible first-order comparison]\label{thm:comparison}
Choose $p(x)=Y_1(x)$ with probability $\lambda$ and $p(x)=P(x)$ otherwise. Define
\begin{equation}\label{eq:combinedfield}
G(f,x)=\lambda\bigl[q_s(f,x)+\ell\nabla f(x)+b_s h(f,x)\bigr].
\end{equation}
For every $x,o\in\K$,
\begin{equation}\label{eq:linearization}
\alpha_* f(o)-\E[f(p(x))\mid f,x]\le\ip{G(f,x)}{o-x}.
\end{equation}
The action distribution depends only on $x$, and all gradients defining $G(f,x)$ are taken in $[0,x]$.
\end{theorem}
\begin{proof}
Write $R_f(x,o)=f(x\odot o)+f(x)$. By \eqref{eq:trajectoryfree} with $\mu_s-\ell=2\ell=b_s$, and by \eqref{eq:boxcompare},
\[
\begin{aligned}
\theta_s f(o)-f(Y_1(x))
&\le b_sR_f(x,o)+\ip{q_s(f,x)+\ell\nabla f(x)}{o-x},\\
R_f(x,o)-\tfrac94 f(P(x))
&\le\ip{h(f,x)}{o-x}.
\end{aligned}
\]
Multiply the first inequality by $\lambda$ and the second by $\lambda b_s$, and add them. The terms $\lambda b_sR_f(x,o)$ on the two sides cancel, and by \eqref{eq:combinedfield} the right side becomes $\ip{G(f,x)}{o-x}$. By $\alpha_*=\lambda\theta_s$ and $\tfrac94\lambda b_s=1-\lambda$, the left side is $\alpha_*f(o)-\lambda f(Y_1(x))-(1-\lambda)f(P(x))$. Since $p(x)=Y_1(x)$ with probability $\lambda$ and $p(x)=P(x)$ otherwise, $\lambda f(Y_1(x))+(1-\lambda)f(P(x))=\E[f(p(x))\mid f,x]$, which proves \eqref{eq:linearization}. Both actions lie in $[0,x]$ by Propositions~\ref{prop:trajectory} and~\ref{prop:box}.
\end{proof}

Inequality \eqref{eq:linearization} has the form of the upper-linearizable comparisons of \citet{pedramfar2024linear,lu2026upper}, except that the action is random with a distribution that depends only on $x$. Section~\ref{sec:regret} shows that conditional expectation carries their regret argument over to this case. The vector $G(f,x)$ need not be the gradient of any function. We randomize between the two actions instead of playing their average because DR-submodularity does not bound $f$ at the average below by the average of the two values.

\section{The online algorithm and its guarantees}\label{sec:implementation}\label{sec:algorithm}
The vector $G(f,x)$ in \eqref{eq:combinedfield} has three components, and computing it exactly would need many gradients. The algorithm instead picks one component at random and estimates it from a single gradient query in $[0,x]$. Lemma~\ref{lem:estimator} shows that the resulting estimate $\widehat G$ of $G(f,x)$ is unbiased and has second moment at most $B^2\Gamma_R^2$, with $\Gamma_R$ defined in \eqref{eq:branchprob}.

\subsection{The one-query estimator and the algorithm}\label{sec:sample-trajectory}
For the trajectory component, let $c_0=1-e^{-1}$, and let $V$ have density $e^{v-1}/c_0$ on $[0,1]$. Such a $V$ is obtained as $V=\log(1+(e-1)Q)$ with $Q$ uniform on $[0,1]$. The trajectory component $q_s(f,x)$ is estimated by $c_0\,\omega_V(x)\odot\nabla f(Y_V(x))$, whose query point $Y_V(x)$ lies in $[0,x]$.

For the box component, use the chains of Lemma~\ref{lem:chain} for $a=\one-x$, built in the sorted order of the weights $w_i$. If $x=0$, then $h(f,x)=0$, and this branch queries $\nabla f(0)$ and returns $\widehat h=0$. Otherwise draw an original coordinate $I$ with probability $\rho_I=x_I^2/\norm{x}_2^2$, let $j$ be its position, so that $\sigma_j=I$, and draw $S$ uniformly from $[0,1]$. With probability $1-r_I$ set $Z=L^{j-1}+Sr_Ie_I$, and otherwise set $Z=U^{j-1}-S(1-r_I)e_I$. One gradient of $f$ gives
\begin{equation}\label{eq:boxestimator}
\widehat h=\frac{x_I}{\rho_I}\,[\nabla f(x\odot(\one-Z))]_I\,e_I=\frac{\norm{x}_2^2}{x_I}\,[\nabla f(x\odot(\one-Z))]_I\,e_I.
\end{equation}
Drawing $I$ with probability $\rho_I$ and dividing by $\rho_I$ in \eqref{eq:boxestimator} is importance sampling \citep{kahn1953methods}, which \citet{needell2016stochastic,zhao2015stochastic} use to reduce the variance of stochastic gradient methods. The weights $\rho_I\propto x_I^2$ make the second moment of $\widehat h$ scale with $\norm{x}_2^2$, as shown in \eqref{eq:boxmoment} in the proof of Lemma~\ref{lem:estimator}. Coordinates with $x_i=0$ are never drawn, and $h_i(f,x)=0$ for them, because every average in \eqref{eq:chainfield} for $J_x$ carries the factor $x_i$. The chains consist of $2d$ segments determined by $x$ alone. Only one random point on one segment is queried, and it lies in $[0,x]$. Apart from the gradient query and the projection, a round takes $O(d\log d)$ arithmetic operations, dominated by the sort.

The middle component $\ell\nabla f(x)$ needs one gradient at $x$. The estimator combines the three components with the branch probabilities
\begin{equation}\label{eq:branchprob}
\Sigma_R=c_0+\ell+\sqrt2\,b_sR,\qquad
(\pi_0,\pi_1,\pi_2)=\frac{(c_0,\ell,\sqrt2\,b_sR)}{\Sigma_R},\qquad \Gamma_R=\lambda \Sigma_R.
\end{equation}
After the action is committed, an independent branch $\nu\in\{0,1,2\}$ is drawn with probabilities $(\pi_0,\pi_1,\pi_2)$, and only the gradient of that branch is queried:
\begin{equation}\label{eq:onequery}
\widehat G=
\begin{cases}
\lambda c_0\,\omega_V(x)\odot\nabla f(Y_V(x))/\pi_0,&\nu=0,\quad V\sim e^{v-1}/c_0,\\
\lambda\ell\nabla f(x)/\pi_1,&\nu=1,\\
\lambda b_s\widehat h/\pi_2,&\nu=2.
\end{cases}
\end{equation}
The query point is $Y_V(x)$, $x$ or $x\odot(\one-Z)$, and all three lie in $[0,x]$.

The next lemma shows that \eqref{eq:onequery} estimates $G(f,x)$ without bias from one feasible query and bounds its second moment.
\begin{lemma}\label{lem:estimator}
For every $f$ satisfying \eqref{eq:assumptions} and every $x\in\K$, the estimator \eqref{eq:onequery} makes one gradient query, at a point of $[0,x]$, and satisfies
\begin{equation}\label{eq:Gmoment}
\E[\widehat G\mid f,x]=G(f,x),\qquad
\E[\norm{\widehat G}_2^2\mid f,x]\le B^2\Gamma_R^2.
\end{equation}
Among fixed positive branch probabilities, those in \eqref{eq:branchprob} minimize the bound $\lambda^2B^2\bigl[c_0^2/\pi_0+\ell^2/\pi_1+2R^2b_s^2/\pi_2\bigr]$ on the second moment.
\end{lemma}
Lemma~\ref{lem:estimator} is proved in Appendix~\ref{app:estimation}.

Algorithm~\ref{alg:main} plays the random action of Theorem~\ref{thm:comparison}, forms $\widehat G$ from one gradient query, and updates its state by projected gradient ascent. The random choice of the action and the random choice of the queried branch are independent.

\begin{algorithm2e}[H]
\caption{Online maximization with one feasible gradient query}\label{alg:main}
\DontPrintSemicolon
\LinesNumbered
\KwIn{Projection oracle $\Pi_\K$, horizon $T$, constants $D,R,B>0$.}
Set the constants in \eqref{eq:parameters} and \eqref{eq:branchprob}, $\eta=D/(B\Gamma_R\sqrt T)$ and $x_1=0$.\;
\For{$t=1,\ldots,T$}{
Without observing $f_t$, commit to $p_t=Y_1(x_t)$ with probability $\lambda$, or $p_t=P(x_t)$ otherwise.\;
Independently of the action, draw the branch and its query point as in \eqref{eq:onequery}.\;
Query the gradient of $f_t$ once at that point and form $\widehat G_t$.\;
$x_{t+1}\gets\Pi_\K(x_t+\eta\widehat G_t)$.\;
}
\end{algorithm2e}

\subsection{Regret bounds}\label{sec:regret}
Applied in round $t$, Theorem~\ref{thm:comparison} bounds the approximation gap of the played action by $\ip{G(f_t,x_t)}{o-x_t}$ for every $o\in\K$. Lemma~\ref{lem:estimator} gives an unbiased estimate of $G(f_t,x_t)$ with bounded second moment. The next lemma, from \citet{aggarwal2026geometry}, bounds the regret of projected gradient ascent driven by such estimates. Let $\mathcal H_t$ contain $f_1,\ldots,f_T$ and all draws and responses before round $t$.

\begin{lemma}[{\citealp[Section~5.2]{aggarwal2026geometry}}]\label{lem:reduction}
Let the adversary be oblivious, and let $\alpha\in(0,1]$ and $\kappa>0$. Suppose the learner plays $p_t\in\K$ and updates $x_{t+1}=\Pi_\K(x_t+\eta\widehat G_t)$, with $x_1\in\K$ and $\eta=D/(\kappa\sqrt T)$. Suppose also that, in every round $t$, $x_t$ is determined by $\mathcal H_t$ and there is a vector $G_t$ determined by $\mathcal H_t$ such that
\begin{equation}\label{eq:reductionconditions}
\alpha f_t(o)-\E[f_t(p_t)\mid\mathcal H_t]\le\ip{G_t}{o-x_t}\quad(o\in\K),\qquad
\E[\widehat G_t\mid\mathcal H_t]=G_t,\qquad
\E[\norm{\widehat G_t}_2^2\mid\mathcal H_t]\le\kappa^2.
\end{equation}
Then $\Reg_\alpha(T)\le D\kappa\sqrt T$.
\end{lemma}
Lemma~\ref{lem:reduction} is the case $\beta=1$, $e_t=0$ and $V=\kappa^2$ of (40)--(41) in \citet{aggarwal2026geometry}. We use Lemma~\ref{lem:reduction} rather than Theorem~1 of \citet{pedramfar2024linear} because Algorithm~\ref{alg:main} violates two assumptions of that theorem. First, Theorem~1 of \citet{pedramfar2024linear} assumes a deterministic action, whereas Algorithm~\ref{alg:main} plays $Y_1(x_t)$ or $P(x_t)$ at random. Second, Theorem~1 of \citet{pedramfar2024linear} assumes that every estimate is bounded almost surely by a constant that does not depend on the state. The box estimate $\widehat h$ in \eqref{eq:boxestimator} has norm $\norm{x}_2^2|\partial_If(x\odot(\one-Z))|/x_I$, which grows without bound as a positive coordinate $x_I$ approaches $0$, so no such uniform bound holds. Its conditional second moment is nevertheless at most $2B^2R^2$ (Appendix~\ref{app:estimation}), and Lemma~\ref{lem:reduction} needs only this second-moment bound.

Lemmas~\ref{lem:estimator} and~\ref{lem:reduction}, together with Theorem~\ref{thm:comparison}, give the regret bound of Algorithm~\ref{alg:main}.

\begin{theorem}\label{thm:regret}
For every integer horizon $T\ge1$ and every oblivious adversary, Algorithm~\ref{alg:main} with exact gradients satisfies
\begin{equation}\label{eq:mainregret}
\Reg_{\alpha_*}(T)\le DB\Gamma_R\sqrt T.
\end{equation}
Here $\Gamma_R=\lambda(c_0+\ell+\sqrt2\,b_sR)=O(1+R)$, and every round uses one gradient query at a point of $\K$ and one projection.
\end{theorem}
\begin{proof}
We verify the conditions of Lemma~\ref{lem:reduction} with $\alpha=\alpha_*$, $G_t=G(f_t,x_t)$ and $\kappa=B\Gamma_R$. The state $x_t$ is determined by $\mathcal H_t$, because the coordinate sort and the projection are measurable maps, and hence so is $G_t$. The play draw of round $t$ depends on the past only through $x_t$, so
\[
\E[f_t(p_t)\mid\mathcal H_t]=\lambda f_t(Y_1(x_t))+(1-\lambda)f_t(P(x_t)).
\]
Theorem~\ref{thm:comparison} at $x=x_t$ then gives the first condition in \eqref{eq:reductionconditions}. The query draws of round $t$ are independent of the play draw given $x_t$, so Lemma~\ref{lem:estimator} gives the other two. Algorithm~\ref{alg:main} starts at $x_1=0\in\K$ with step size $\eta=D/(B\Gamma_R\sqrt T)=D/(\kappa\sqrt T)$. Lemma~\ref{lem:reduction} gives $\Reg_{\alpha_*}(T)\le D\kappa\sqrt T=DB\Gamma_R\sqrt T$.
\end{proof}
Since $D\le\sqrt2R$ by Section~\ref{sec:setup} and $\Gamma_R=O(1+R)$, the bound \eqref{eq:mainregret} is $O(RB(1+R)\sqrt T)$. Since $R\le\sqrt d$, it is also at most $DB\lambda(c_0+\ell+b_s\sqrt{2d})\sqrt T=O(DB\sqrt{dT})$. For the budget set $\{x\in[0,1]^d:\sum_ix_i\le k\}$ with $k\ge1$, Section~\ref{sec:setup} allows $R=\sqrt k$, so $D\le\sqrt{2k}$ and the bound is $O(kB\sqrt T)$. For the simplex, $k=1$, and the bound is at most $\sqrt2\,\Gamma_1B\sqrt T\le0.96\,B\sqrt T$.

The same reduction applies when each query returns an unbiased random estimate of the gradient instead of the gradient itself. Suppose that a query at $z$ in round $t$ returns a random vector $\widetilde\nabla f_t(z)$ with
\begin{equation}\label{eq:stochasticassumption}
\E[\widetilde\nabla f_t(z)\mid\mathcal P_t]=\nabla f_t(z),\qquad
\E[\norm{\widetilde\nabla f_t(z)}_2^2\mid\mathcal P_t]\le B_1^2,
\end{equation}
where $\mathcal P_t$ contains $\mathcal H_t$, the play draw of round $t$ and the draws that determine the query point. The bound \eqref{eq:boxmoment} on the box branch uses the monotonicity of exact gradients on $[0,x]$, which noisy responses need not have. The stochastic version therefore draws the coordinate of the box branch with probability proportional to $x_I$ instead of $x_I^2$. Let $\Lambda\ge\sup_{x\in\K}\norm{x}_1$ be known to the learner, for example $\Lambda=d$ in general and $\Lambda=k$ for the budget set, and put
\begin{equation}\label{eq:stochasticconstants}
\widetilde\Gamma_\Lambda=\lambda(c_0+\ell+b_s\Lambda),\qquad
(\widetilde\pi_0,\widetilde\pi_1,\widetilde\pi_2)=\frac{(c_0,\ell,b_s\Lambda)}{c_0+\ell+b_s\Lambda}.
\end{equation}

\begin{corollary}[Stochastic gradients]\label{cor:stochastic}
Let the adversary be oblivious and let the responses satisfy \eqref{eq:stochasticassumption} with a known $B_1>0$. Modify Algorithm~\ref{alg:main} in four ways. Each gradient in \eqref{eq:onequery} is replaced by the response $\widetilde\nabla f_t$ at the same point. The branch probabilities are $(\widetilde\pi_0,\widetilde\pi_1,\widetilde\pi_2)$ in place of $(\pi_0,\pi_1,\pi_2)$. For $x\ne0$, the box branch draws $I$ with probability $\rho_I=x_I/\norm{x}_1$ and returns $\widehat h=\norm{x}_1[\widetilde\nabla f_t(z)]_Ie_I$, where $z=x\odot(\one-Z)$ is chosen as in \eqref{eq:boxestimator}. The step size is $\eta=D/(B_1\widetilde\Gamma_\Lambda\sqrt T)$. Then
\begin{equation}\label{eq:stochasticregret}
\Reg_{\alpha_*}(T)\le DB_1\widetilde\Gamma_\Lambda\sqrt T,
\end{equation}
and every round uses one stochastic gradient query at a point of $[0,x_t]\subseteq\K$ and one projection.
\end{corollary}
\begin{proof}
Corollary~\ref{cor:stochastic} follows from Lemma~\ref{lem:reduction} with $\alpha=\alpha_*$, $G_t=G(f_t,x_t)$ and $\kappa=B_1\widetilde\Gamma_\Lambda$, so it suffices to check the three conditions in \eqref{eq:reductionconditions}. The action and the update are those of Algorithm~\ref{alg:main}, so the first condition holds as in the proof of Theorem~\ref{thm:regret}.

For the second condition, each branch output is a fixed linear function of the response at its query point, so by the first identity in \eqref{eq:stochasticassumption} its conditional mean given $\mathcal P_t$ is the output with the exact gradient $\nabla f_t(z)$. Appendix~\ref{app:estimation} computes the means of these exact outputs: \eqref{eq:trajmean} for the trajectory branch and \eqref{eq:boxmean} for the box branch, whose derivation uses only that $\rho_i>0$ exactly when $x_i>0$. Averaging over the branch, drawn with probability $\widetilde\pi_\nu$ and divided by $\widetilde\pi_\nu$, gives $\E[\widehat G_t\mid\mathcal H_t]=G_t$ by \eqref{eq:combinedfield}.

For the third condition, the second inequality in \eqref{eq:stochasticassumption} and $0\le\omega_V\le\one$ bound the unscaled second moments of the trajectory and endpoint branches by $c_0^2B_1^2$ and $\ell^2B_1^2$. The box output has the single nonzero coordinate $\norm{x_t}_1[\widetilde\nabla f_t(z)]_I$, so its second moment is at most $\norm{x_t}_1^2B_1^2\le\Lambda^2B_1^2$. Conditioning on the branch as in the proof of Lemma~\ref{lem:estimator} and using \eqref{eq:stochasticconstants},
\[
\E[\norm{\widehat G_t}_2^2\mid\mathcal H_t]\le\lambda^2B_1^2\Bigl[\frac{c_0^2}{\widetilde\pi_0}+\frac{\ell^2}{\widetilde\pi_1}+\frac{b_s^2\Lambda^2}{\widetilde\pi_2}\Bigr]=\lambda^2B_1^2(c_0+\ell+b_s\Lambda)^2=B_1^2\widetilde\Gamma_\Lambda^2.
\]
\end{proof}
The coefficient $\alpha_*$ and the $\sqrt T$ dependence are the same as in Theorem~\ref{thm:regret}. Since $\widetilde\Gamma_\Lambda=O(1+\Lambda)$, the bound \eqref{eq:stochasticregret} is $O(DdB_1\sqrt T)$ in general and $O(k^{3/2}B_1\sqrt T)$ for the budget set with $k\ge1$, because $D\le\sqrt{2k}$ there.

Algorithm~\ref{alg:main} queries the gradient at a point of $[0,x_t]$ that generally differs from its action $p_t$. Under semi-bandit feedback, defined in Section~\ref{sec:setup}, the learner may query the gradient only at the point it plays. The conversion SFTT of \citet[Algorithm~4]{pedramfar2024linear} turns an algorithm with one query per round into an algorithm with semi-bandit feedback. It splits the rounds into blocks of $L$ consecutive rounds and runs one round of the base algorithm per block. At the start of a block, it fixes the base algorithm's state, action and query point and draws one round of the block uniformly at random. In that round it plays the query point and observes the gradient there, in the other rounds of the block it plays the action, and after the block it passes the observed gradient to the base algorithm. The observed gradient is the exact gradient of the function of a random round of the block. Given that round, the estimator \eqref{eq:onequery} is the estimator of Lemma~\ref{lem:estimator} for that function, so Algorithm~\ref{alg:main} itself can serve as the base algorithm. SFTT plays query points as actions, so it applies only when every query lies in $\K$. Our algorithm satisfies this, because all its queries lie in $[0,x_t]\subseteq\K$, whereas the cube-query algorithms of \citet{aggarwal2026geometry,aggarwal2026du} query points that generally lie outside $\K$.

\begin{corollary}[Semi-bandit feedback]\label{cor:semibandit}
Let the adversary be oblivious, and let $L\ge2$ be an integer that divides $T$. Let $\mathcal A$ be Algorithm~\ref{alg:main} with horizon $T/L$, and let $\mathcal A'$ be SFTT applied to $\mathcal A$ with block length $L$. Then $\mathcal A'$ uses semi-bandit feedback, every point it plays lies in $\K$, and
\begin{equation}\label{eq:semibanditregret}
\Reg_{\alpha_*}(T)\le DB\Bigl[\frac TL+\Gamma_R\sqrt{TL}\Bigr].
\end{equation}
In particular, if $T=L^3$, then $\Reg_{\alpha_*}(T)\le DB(1+\Gamma_R)T^{2/3}$.
\end{corollary}
\begin{proof}
We apply Theorem~7 of \citet{pedramfar2024linear} on the interval $[1,T]$, with one query per round and first-order feedback. Let $\mathcal F_B$ be the class of nonnegative, continuously differentiable DR-submodular functions on $[0,1]^d$ with $\sup_{z\in\K}\norm{\nabla f(z)}_2\le B$. Each of these properties is preserved by convex combinations, so $\mathcal F_B$ is closed under convex combinations. Each $f\in\mathcal F_B$ is $B$-Lipschitz on the convex set $\K$, so that theorem applies with Lipschitz constant $M_1=B$. The query of $\mathcal A$ in a round is drawn before the response of that round is observed. Because $L$ divides $T$, the interval $[1,T]$ consists of $T/L$ complete blocks, so the term $2M_1D(L-1)$ of that theorem, which accounts for incomplete blocks, is zero. Theorem~7 of \citet{pedramfar2024linear} then gives
\[
\Reg_{\alpha_*}(T)\le BD\,\frac TL+L\,R_{\mathcal A}(T/L),
\]
where $R_{\mathcal A}(T/L)$ is the $\alpha_*$-regret of $\mathcal A$ over $T/L$ rounds against the block adversary in the proof of that theorem. Its function in block $q$ is the average $\bar f_q$ of the $L$ functions of the block, which lies in $\mathcal F_B$. Its oracle answers a query at $z$ in block $q$ with $\nabla f_\tau(z)$, where $\tau$ is a round of block $q$ drawn uniformly and independently of the draws of $\mathcal A$.

We bound $R_{\mathcal A}(T/L)$ by Lemma~\ref{lem:reduction}, with rounds indexed by blocks, $\alpha=\alpha_*$, $G_q=G(\bar f_q,x_q)$ and $\kappa=B\Gamma_R$. The first condition in \eqref{eq:reductionconditions} follows from Theorem~\ref{thm:comparison} applied to $\bar f_q$, as in the proof of Theorem~\ref{thm:regret}. Conditional on $\tau=t$, the estimate $\widehat G_q$ is the estimator \eqref{eq:onequery} for $f_t$, and $f_t$ satisfies \eqref{eq:assumptions}. Lemma~\ref{lem:estimator} therefore gives $\E[\widehat G_q\mid\mathcal H_q,\tau=t]=G(f_t,x_q)$ and $\E[\norm{\widehat G_q}_2^2\mid\mathcal H_q,\tau=t]\le B^2\Gamma_R^2$. The vector $G(f,x)$ in \eqref{eq:combinedfield} is linear in $f$, because $q_s(f,x)$, $\nabla f(x)$ and $h(f,x)$ are. Averaging over $\tau$ therefore gives $\E[\widehat G_q\mid\mathcal H_q]=G(\bar f_q,x_q)$ and $\E[\norm{\widehat G_q}_2^2\mid\mathcal H_q]\le B^2\Gamma_R^2$. Lemma~\ref{lem:reduction} with horizon $T/L$ and the step size $\eta=D/(B\Gamma_R\sqrt{T/L})$ of Algorithm~\ref{alg:main} gives $R_{\mathcal A}(T/L)\le DB\Gamma_R\sqrt{T/L}$. Substituting this bound proves \eqref{eq:semibanditregret}. If $T=L^3$, then $T/L=\sqrt{TL}=T^{2/3}$, which gives the second bound.

Finally, SFTT plays only actions and query points of $\mathcal A$. These lie in $[0,x]\subseteq\K$ by Section~\ref{sec:trajectory}, Proposition~\ref{prop:box} and Lemma~\ref{lem:estimator}. In every round, $\mathcal A'$ observes a gradient only at the point it plays, which is semi-bandit feedback.
\end{proof}
The coefficient $\alpha_*$ is the same as in Theorem~\ref{thm:regret}. The regret grows as $T^{2/3}$ instead of $\sqrt T$ because, in one round per block, the learner plays the base algorithm's query point instead of its action. For the budget set with $k\ge1$, $D\le\sqrt{2k}$ and $\Gamma_R=O(\sqrt k)$, so the bound is $O(kBT^{2/3})$. When $L$ does not divide $T$, Remark~2 of \citet{pedramfar2024linear} runs SFTT on the first $L\lfloor T/L\rfloor$ rounds and plays a fixed point of $\K$ in the remaining rounds, which adds at most $BD(L-1)$ to the regret.

The conversion FOTZO of \citet[Algorithm~5]{pedramfar2024linear} replaces each gradient query by one value query. Let $I_\K=\{i\in[d]:x_i>0\text{ for some }x\in\K\}$ and $m=|I_\K|$. Every point of $\K$ vanishes outside $I_\K$, so we identify $\K$ with a subset of $[0,1]^m$ and each $f_t$ with its restriction to the face of the cube on which the coordinates outside $I_\K$ vanish. This restriction is nonnegative, continuously differentiable and DR-submodular, and its gradient on $\K$ has norm at most $B$. Because $\K$ is down-closed and convex, it contains a positive multiple of each $e_i$ with $i\in I_\K$, so $\K$ has nonempty interior in $\R^m$. Let $\mathbb B_r(c)$ denote the closed Euclidean ball in $\R^m$ of radius $r$ around $c$, and suppose the learner is given $c\in\K$ and $r>0$ with $\mathbb B_r(c)\subseteq\K$. For $0<\delta<r$, put $a_\delta=1-\delta/r$ and $\Phi_\delta(z)=a_\delta z+(1-a_\delta)c$. FOTZO runs the base algorithm on $\K$ and plays the base algorithm's action unchanged. It replaces a gradient query at $z$ by a value query at $\Phi_\delta(z)+\delta v$, with $v$ uniform on the unit sphere of $\R^m$, and passes the vector $a_\delta(m/\delta)\,y\,v$ to the base algorithm, where $y$ is the returned value. Averaged over $v$ and the noise of the value oracle, this vector equals $\nabla f'_\delta(z)$ \citep[proof of Theorem~5]{pedramfar2024linear}, where
\begin{equation}\label{eq:smoothed}
f'_\delta(x)=\E_u\,f\bigl(a_\delta x+(1-a_\delta)c+\delta u\bigr),\qquad x\in[0,1]^m,
\end{equation}
and $u$ is uniform on the unit ball of $\R^m$. The query point $\Phi_\delta(z)+\delta v=a_\delta z+(1-a_\delta)(c+rv)$ is a convex combination of $z$ and a point of $\mathbb B_r(c)$, so it lies in $\K$ because $z\in\K$. Composing FOTZO with SFTT gives bandit feedback, under which the learner observes only the value $f_t(p_t)$ at the point it plays.

\begin{corollary}[Zeroth-order and bandit feedback]\label{cor:zeroth}
Let the adversary be oblivious, and let $m$, $c$ and $r$ be as above. Suppose that a value query at $z\in\K$ in round $t$ returns a random number with conditional mean $f_t(z)$ and absolute value at most $B_0$, where $B_0>0$ is known to the learner. Let $\mathcal A$ be the algorithm of Corollary~\ref{cor:stochastic} in dimension $m$ with $B_1=mB_0/\delta$, run for $S$ rounds, where $\delta=\min\{r/2,S^{-1/4}\}$. Then the following hold, with constants that depend on $m$, $D$, $r$, $\Lambda$, $B$ and $B_0$.
\begin{enumerate}
\item With one value query per round, FOTZO applied to $\mathcal A$ with $S=T$ satisfies $\Reg_{\alpha_*}(T)=O(T^{3/4})$.
\item Under bandit feedback, SFTT applied to FOTZO$(\mathcal A)$, with the block length $L$ of Corollary~6 of \citet{pedramfar2024linear} and $S=\lfloor T/L\rfloor$, satisfies $\Reg_{\alpha_*}(T)=O(T^{4/5})$.
\end{enumerate}
Every point queried or played lies in $\K$.
\end{corollary}
\begin{proof}
\emph{Part~1.} Theorem~5 of \citet{pedramfar2024linear} bounds the regret of FOTZO$(\mathcal A)$ by the regret of $\mathcal A$ on the functions $f'_{t,\delta}$ defined by \eqref{eq:smoothed} with $f=f_t$, with the stochastic first-order oracle described above, plus $(3+2D/r)\delta BT$. Each response of this oracle has norm at most $mB_0/\delta$, because $|y|\le B_0$, $a_\delta\le1$ and $\norm{v}_2=1$. We check that each $f'_{t,\delta}$ satisfies \eqref{eq:assumptions} in dimension $m$, so that Corollary~\ref{cor:stochastic} applies to it. Fix $u$ in the unit ball of $\R^m$. The point $a_\delta x+(1-a_\delta)c+\delta u=a_\delta x+(1-a_\delta)(c+ru)$ is a convex combination of $x$ and $c+ru\in\mathbb B_r(c)\subseteq\K$, so it lies in $[0,1]^m$ whenever $x$ does. The map $x\mapsto a_\delta x+(1-a_\delta)c+\delta u$ is increasing with slope $a_\delta>0$. Hence $x\mapsto f_t(a_\delta x+(1-a_\delta)c+\delta u)$ is nonnegative and continuously differentiable on $[0,1]^m$, and its gradient $a_\delta\nabla f_t(a_\delta x+(1-a_\delta)c+\delta u)$ is nonincreasing in $x$, so it is DR-submodular. Averaging over $u$ keeps these three properties. For $x\in\K$, every point $a_\delta x+(1-a_\delta)c+\delta u$ lies in $\K$, so $\norm{\nabla f'_{t,\delta}(x)}_2\le a_\delta B\le B$. Thus $f'_{t,\delta}$ satisfies \eqref{eq:assumptions}, and Corollary~\ref{cor:stochastic} with $B_1=mB_0/\delta$ bounds the regret of $\mathcal A$ by $D(mB_0/\delta)\widetilde\Gamma_\Lambda\sqrt T$. With $\delta=\min\{r/2,T^{-1/4}\}$, the two terms are $O(T^{1/4}\sqrt T)$ and $O(T^{-1/4}T)$, so the regret is $O(T^{3/4})$. This is Corollary~4 of \citet{pedramfar2024linear} with base exponent $1/2$.

\emph{Part~2.} FOTZO$(\mathcal A)$ makes one value query per round, at a point of $\K$ chosen before the value is observed. By part~1, its regret over $S$ rounds is $O(S^{3/4})$, uniformly over oblivious adversaries whose functions lie in the class $\mathcal F_B$ of the proof of Corollary~\ref{cor:semibandit} and whose value oracles are bounded by $B_0$. Theorem~7 of \citet{pedramfar2024linear} with zeroth-order feedback therefore applies to SFTT applied to FOTZO$(\mathcal A)$. Corollary~6 there takes $L=\Theta(T^{1/5})$, handles the rounds after the last complete block as in Remark~2 there, and gives regret $O(T^{1/(2-3/4)})=O(T^{4/5})$. SFTT plays only actions and query points of FOTZO$(\mathcal A)$, all of which lie in $\K$, and it observes a value only at the point it plays, which is bandit feedback.
\end{proof}

\subsection{Offline guarantee}\label{sec:extensions}
The online-to-batch conversion OTB of \citet{pedramfar2024linear} runs an online algorithm on the constant sequence $f_t=f$ and returns a uniformly sampled played action. Applied to Algorithm~\ref{alg:main}, it gives the following offline guarantee, in which selecting the output needs no function values.

\begin{corollary}[Offline guarantee by OTB]\label{cor:offline}
For one fixed objective $f$, let $\OPT=\max_{o\in\K}f(o)$. For every integer $N\ge1$, run Algorithm~\ref{alg:main} with horizon $N$ and $f_t=f$, and return $p_\tau$ for an independent uniform index $\tau\in[N]$. Then
\begin{equation}\label{eq:offline}
\E f(p_\tau)\ge\alpha_*\OPT-\frac{DB\Gamma_R}{\sqrt N}.
\end{equation}
Consequently, the guarantee $\E f(p_\tau)\ge\alpha_*\OPT-\varepsilon$ is attained for every $\varepsilon>0$ using
\begin{equation}\label{eq:accuracybudget}
N_\varepsilon=\max\!\left\{1,\left\lceil\frac{D^2B^2\Gamma_R^2}{\varepsilon^2}\right\rceil\right\}
\end{equation}
feasible gradient queries and the same number of projections.
\end{corollary}
\begin{proof}
We apply OTB to Theorem~\ref{thm:regret}. For $f_t=f$ we have $\max_{o\in\K}\sum_{t=1}^Nf_t(o)=N\,\OPT$, so \eqref{eq:mainregret} with $T=N$ gives $\E\sum_{t=1}^Nf(p_t)\ge N\alpha_*\OPT-DB\Gamma_R\sqrt N$. Since $\tau$ is uniform on $[N]$ and independent of the run, $\E f(p_\tau)=\frac1N\E\sum_{t=1}^Nf(p_t)$, and dividing by $N$ gives \eqref{eq:offline}. If $N\ge D^2B^2\Gamma_R^2/\varepsilon^2$, then $DB\Gamma_R/\sqrt N\le\varepsilon$, which gives \eqref{eq:accuracybudget}. Each round uses one gradient query and one projection.
\end{proof}
Corollary~\ref{cor:offline} gives a slightly larger offline coefficient than \citet{buchbinder2024bounds}, with every query in $\K$. 

\section{Conclusion}\label{sec:conclusion}
We gave an algorithm with coefficient $0.4013$ for adversarial online non-monotone DR-submodular maximization over down-closed convex sets in which every gradient query lies in the feasible set, using one query and one projection per round. To our knowledge, it is the first online coefficient above $1/e$ under this restriction. With queries anywhere in the cube, $4/9$ \citep{aggarwal2026geometry} and $2\sqrt3-3$ \citep{aggarwal2026du} are known. The key step applies the coordinate-chain comparison of \citet{aggarwal2026geometry} inside the box below the state, which keeps every query feasible, and cancels the values it leaves behind against a delayed-path bound through a fixed mixture of the two actions. The guarantee needs only continuous differentiability and a gradient bound on $\K$, allows feasible sets with empty interior, and extends to stochastic gradients, semi-bandit, zeroth-order and bandit feedback, and the offline problem. Whether $4/9$ or more can be attained with feasible queries is open.
\label{lastmainpage}

\section*{AI use statement}
Generative AI tools were used to edit the paper. The authors reviewed all AI-assisted text and take responsibility for the final content.

\appendix

\section{Auxiliary lemmas for the proof of Proposition~\ref{prop:trajectory}}\label{app:trajectory}
This appendix proves the lemmas used in Appendix~\ref{app:trajectory-integration}. Every function below is nonnegative, continuously differentiable and DR-submodular on its cube, and $x,o\in\K$. We write $Y_v=Y_v(x)$ and $\omega_v=\omega_v(x)$ for the path and weights in \eqref{eq:path} and \eqref{eq:omega}, and we set $C_v=Y_v+\omega_v\odot o$. The points $C_v$, $U=x\oplus o$, $H$ and $W$ below appear only in the analysis and are never queried.

\subsection{Directional inequalities and the exponential expansion}\label{app:trajectory-preliminaries}
Let $u\le v$ and $\Delta\ge0$ with $v+\Delta\le\one$. For $\tau\in[0,1]$, $u+\tau\Delta\le v+\tau\Delta$, so $\nabla f(u+\tau\Delta)\ge\nabla f(v+\tau\Delta)$, and since $\Delta\ge0$, $\ip{\nabla f(u+\tau\Delta)}{\Delta}\ge\ip{\nabla f(v+\tau\Delta)}{\Delta}$. Integrating over $\tau\in[0,1]$ gives
\begin{equation}\label{eq:drincrements}
f(u+\Delta)-f(u)\ge f(v+\Delta)-f(v).
\end{equation}
Next let $\Delta\ge0$. The function $\phi(\tau)=f(u+\tau\Delta)$ has derivative $\phi'(\tau)=\ip{\nabla f(u+\tau\Delta)}{\Delta}$, which is nonincreasing in $\tau$ because $u+\tau\Delta$ increases with $\tau$. Hence $\phi$ is concave, so $\phi(1)-\phi(0)\le\phi'(0)$ and $\phi(0)-\phi(-1)\ge\phi'(0)$ whenever the points involved lie in the cube. In terms of $f$,
\begin{equation}\label{eq:drtangents}
\ip{\nabla f(u)}{\Delta}\ge f(u+\Delta)-f(u),\qquad
\ip{\nabla f(u)}{\Delta}\le f(u)-f(u-\Delta),
\end{equation}
whenever the displayed points lie in the cube \citep{bian2017structure}.

The proof of Proposition~\ref{prop:trajectory} rests on the following exponential expansion of \citet{buchbinder2024bounds}. The expansion bounds $F$ at the point reached from $\one-a$ by an exponential flow along an integrable trajectory.
\begin{lemma}[{\citealp[Lemma~4.1]{buchbinder2024bounds}}]\label{lem:exponential}
For $a\in[0,1]^n$, $t\ge0$, and an integrable $\xi:[0,t]\to[0,1]^n$, let $y(t)=\int_0^t\xi(v)\,dv$. Then
\begin{equation}\label{eq:exponential}
\begin{split}
F(\one-a\odot e^{-y(t)})\ge e^{-t}\biggl[ F(\one-a)
+\sum_{k=1}^\infty\frac1{k!}\int_{[0,t]^k}
F\left((\one-a)\oplus\bigoplus_{j=1}^k\xi(v_j)\right)\,dv\biggr].
\end{split}
\end{equation}
\end{lemma}
The series converges absolutely: $F$ is continuous on the compact cube, so each integrand is at most $\max F$, and the $k$-th term is at most $t^k\max F/k!$.

\subsection{The exponential expansion for two blocks}\label{app:trajectory-twoblock}

We use Lemma~\ref{lem:exponential} in its two-block form. The coordinates are split by weights $b$ and $\one-b$, and each part follows its own trajectory.
\begin{lemma}[{\citealp[Lemma~4.2]{buchbinder2024bounds}}]\label{lem:twoblock}
Let $b,a^{(1)},a^{(2)}\in[0,1]^d$ and let $\xi^{(i)}:[0,t]\to[0,1]^d$ be integrable, with $y^{(i)}=\int_0^t\xi^{(i)}(v)\,dv$. Then
\begin{equation}\label{eq:twoblock}
\begin{split}
&f\bigl(\one-b\odot a^{(1)}\odot e^{-y^{(1)}}-(\one-b)\odot a^{(2)}\odot e^{-y^{(2)}}\bigr)\\
&\quad\ge e^{-t}\left[f(\one-b\odot a^{(1)}-(\one-b)\odot a^{(2)})
+\sum_{k\ge1}\frac1{k!}\int_{[0,t]^k}\Xi_k(v)\,dv\right],
\end{split}
\end{equation}
where
\begin{equation}\label{eq:twointegrand}
\Xi_k(v)=f\left(b\odot\left[(\one-a^{(1)})\oplus\bigoplus_{j=1}^k\xi^{(1)}(v_j)\right]
+(\one-b)\odot\left[(\one-a^{(2)})\oplus\bigoplus_{j=1}^k\xi^{(2)}(v_j)\right]\right).
\end{equation}
\end{lemma}
Lemma~\ref{lem:twoblock} is Lemma~4.2 of \citet{buchbinder2024bounds} with two blocks, weights $b^{(1)}=b$ and $b^{(2)}=\one-b$, scalings $a^{(1)},a^{(2)}$ and trajectories $\xi^{(1)},\xi^{(2)}$.

\subsection{A lower bound on \texorpdfstring{$f(C_v)$}{f(C\_v)} before the delay}\label{app:trajectory-early}
Comparing \eqref{eq:path} with \eqref{eq:omega} also gives
\begin{equation}\label{eq:omegapath}
\omega_v=\one-Y_v-x\quad(v<s),\qquad
\omega_v=\one-Y_v\quad(v\ge s).
\end{equation}
Recall $C_v=Y_v+\omega_v\odot o$. For $v<s$, \eqref{eq:omegapath} and $o\le\one$ give $Y_v\le C_v\le Y_v+(\one-Y_v-x)=\one-x$. For $v\ge s$, \eqref{eq:omegapath} gives $C_v=Y_v+(\one-Y_v)\odot o=Y_v\oplus o\le\one$. In both phases $Y_v\le C_v\le\one$.

The next lemma bounds $f$ at a point below $o$ that is shifted in the direction $\one-x$. Lemma~\ref{lem:comparator} controls the integrands of the expansion in both phases.
\begin{lemma}\label{lem:comparator}
For $x,o\in[0,1]^d$, $0\le\Delta\le\one-o$, and $0\le o'\le o$,
\begin{equation}\label{eq:comparator}
f(o'+(\one-x)\odot\Delta)\ge f(o)+f(0)-f(x\oplus o)-f(o-o').
\end{equation}
\end{lemma}
\begin{proof}
Apply \eqref{eq:drincrements} with $u=0$, $v=o-o'\ge0$ and increment $o'+(\one-x)\odot\Delta\ge0$. The required condition holds because $(o-o')+o'+(\one-x)\odot\Delta\le o+\Delta\le\one$. This gives
\[
f(o'+(\one-x)\odot\Delta)-f(0)
\ge f(o+(\one-x)\odot\Delta)-f(o-o').
\]
Next apply \eqref{eq:drincrements} with $u=o$, $v=x\oplus o\ge o$ and increment $(\one-x)\odot\Delta$. The required condition holds because $\one-(x\oplus o)=(\one-x)\odot(\one-o)\ge(\one-x)\odot\Delta$. Together with $f\ge0$, this gives
\[
f(o+(\one-x)\odot\Delta)-f(o)
\ge f(x\oplus o+(\one-x)\odot\Delta)-f(x\oplus o)\ge-f(x\oplus o).
\]
Adding the two displays and moving $f(0)$ to the right side gives \eqref{eq:comparator}.
\end{proof}

Lemmas~\ref{lem:twoblock} and~\ref{lem:comparator} give the bound for the early phase $v<s$.
\begin{lemma}\label{lem:early}
For $0\le v<s$,
\[
f(C_v)\ge f(o)-f(x\odot o)-(1-e^{-v})f(x\oplus o).
\]
\end{lemma}
\begin{proof}
Here
\[
C_v=(\one-x)-(\one-x)\odot(\one-o)\odot e^{-vx}.
\]
Apply Lemma~\ref{lem:twoblock} on $[0,v]$ with $b=x$, $a^{(1)}=\one$, $\xi^{(1)}=0$, $a^{(2)}=\one-o$ and $\xi^{(2)}=x$, so that $y^{(1)}=0$ and $y^{(2)}=vx$. The left side of \eqref{eq:twoblock} is then $f(C_v)$, and the base term is $f(\one-x-(\one-x)\odot(\one-o))=f((\one-x)\odot o)$. By \eqref{eq:drincrements} with $u=0$, $v=x\odot o$ and increment $(\one-x)\odot o$, whose sum with $x\odot o$ is $o\le\one$, and by $f(0)\ge0$,
\[
f((\one-x)\odot o)\ge f(o)-f(x\odot o)+f(0)\ge f(o)-f(x\odot o).
\]
For $k\ge1$, put $W_k=\bigoplus_{j=1}^kx$. Since $\xi^{(1)}=0$ and $\one-a^{(1)}=0$, the first block of $\Xi_k$ is $0$ and the second is $o\oplus W_k$, so $\Xi_k(v)=f((\one-x)\odot(o\oplus W_k))$. Now $(\one-x)\odot(o\oplus W_k)=o'+(\one-x)\odot\Delta$ with $o'=(\one-x)\odot o$ and $\Delta=(\one-o)\odot W_k$, where $0\le o'\le o$, $o-o'=x\odot o$ and $0\le\Delta\le\one-o$. Hence \eqref{eq:comparator} and $f(0)\ge0$ give
$\Xi_k(v)\ge f(o)-f(x\oplus o)-f(x\odot o)$, a bound that does not depend on the integration times. Replacing each integrand in \eqref{eq:twoblock} by this bound keeps the inequality, because each integral enters with a positive weight. Using $\sum_{k\ge1}v^k/k!=e^v-1$, we obtain
\begin{align*}
f(C_v)&\ge e^{-v}\bigl[f(o)-f(x\odot o)
 +(e^v-1)(f(o)-f(x\oplus o)-f(x\odot o))\bigr]\\
&=f(o)-f(x\odot o)-(1-e^{-v})f(x\oplus o).
\end{align*}

\end{proof}

\subsection{A lower bound on \texorpdfstring{$f(C_v)$}{f(C\_v)} after the delay}\label{app:trajectory-late}

For the late phase $v\ge s$, the expansion also keeps the order-two term described after Proposition~\ref{prop:trajectory}.
\begin{lemma}\label{lem:late}
For $s\le v\le1$,
\[
f(C_v)\ge e^{-v}\left[e^sf(o)-(e^s-1)f(x\oplus o)
 +\left(1+\frac{3s}{4}\right)(v-s)f(x\oplus o)\right].
\]
\end{lemma}
\begin{proof}
For $v\ge s$, $C_v=Y_v\oplus o$ by Appendix~\ref{app:trajectory-early}. Set $\xi'(u)=0$ for $u<s$ and $\xi'(u)=x$ for $u\ge s$, and apply Lemma~\ref{lem:twoblock} on $[0,v]$ with $b=x$, $a^{(1)}=a^{(2)}=\one-o$, $\xi^{(1)}=\xi'$ and $\xi^{(2)}=x$, so that $y^{(1)}=(v-s)x$ and $y^{(2)}=vx$. By \eqref{eq:path}, $\one-Y_v=x\odot e^{-(v-s)x}+(\one-x)\odot e^{-vx}$, and $\one-C_v=(\one-Y_v)\odot(\one-o)$, so
\[
\one-C_v=x\odot(\one-o)\odot e^{-(v-s)x}
 +(\one-x)\odot(\one-o)\odot e^{-vx}.
\]
Hence the left side of \eqref{eq:twoblock} is $f(C_v)$, and its base term is $f(\one-x\odot(\one-o)-(\one-x)\odot(\one-o))=f(o)$. The order-$k$ integrand is
\begin{equation}\label{eq:mainintegrand}
\Xi_k(u)=f\left(x\odot\left[o\oplus\bigoplus_{j=1}^k\xi'(u_j)\right]
 +(\one-x)\odot\left[o\oplus\bigoplus_{j=1}^k x\right]\right).
\end{equation}
We bound $\Xi_k$ on three kinds of regions of $[0,v]^k$. First, if all $k$ times lie in $[0,s)$, then $\xi'(u_j)=0$ for every $j$, so $\Xi_k(u)=f(x\odot o+(\one-x)\odot(o\oplus W_k))=f(o+(\one-x)\odot(\one-o)\odot W_k)$, where $W_k=\bigoplus_{j=1}^kx$. By \eqref{eq:comparator} with $o'=o$ and $\Delta=(\one-o)\odot W_k$, this is at least $f(o)-f(x\oplus o)$. This region has volume $s^k$, so these regions contribute $\sum_{k\ge1}\frac{s^k}{k!}\bigl(f(o)-f(x\oplus o)\bigr)=(e^s-1)\bigl(f(o)-f(x\oplus o)\bigr)$.

Second, put $U=x\oplus o$. If $k=1$ and $u_1\in[s,v]$, both blocks equal $o\oplus x$, so $\Xi_1(u)=f(U)$. This region has length $v-s$ and contributes $(v-s)f(U)$. Third, if $k=2$ and exactly one of the two times lies in $[0,s)$, then the first block is $o\oplus x$, the second is $o\oplus x\oplus x$, and
\begin{align*}
&f\bigl(x\odot(o\oplus x)+(\one-x)\odot(o\oplus(x\oplus x))\bigr)\\
&\qquad=f\bigl(o+(\one-o)\odot[x^2+(\one-x)\odot(2x-x^2)]\bigr)\\
&\qquad=f\bigl(o+(\one-o)\odot[x+x\odot(\one-x)^2]\bigr)=f(H).
\end{align*}
Here
\[
H=U+(\one-o)\odot x\odot(\one-x)^2,
\qquad \one-U=(\one-o)\odot(\one-x).
\]
Hence $H-U=x\odot(\one-x)\odot(\one-U)$, and $x_i(1-x_i)\le1/4$ gives
\[
0\le H-U\le\tfrac14(\one-U).
\]
Therefore $W=U+4(H-U)$ satisfies $U\le W\le\one$, and $H=\tfrac34U+\tfrac14W$. Since $f$ is concave along the nonnegative direction $W-U$ (Appendix~\ref{app:trajectory-preliminaries}) and $f(W)\ge0$,
\[
f(H)\ge\tfrac34f(U)+\tfrac14f(W)\ge\tfrac34f(U).
\]
This region has volume $2s(v-s)$ and enters with weight $1/2!$, so it contributes at least $\tfrac{3s}{4}(v-s)f(U)$. The three kinds of regions are disjoint, and on all remaining regions $\Xi_k\ge0$ because $f\ge0$. Adding the base term $f(o)$ and the three contributions inside the bracket of \eqref{eq:twoblock} gives $f(C_v)\ge e^{-v}\bigl[f(o)+(e^s-1)(f(o)-f(U))+(v-s)f(U)+\tfrac{3s}{4}(v-s)f(U)\bigr]$, which is the claim.
\end{proof}

\section{Proof of Proposition~\ref{prop:trajectory}}\label{app:trajectory-integration}

\begin{proof}[Proof of Proposition~\ref{prop:trajectory}]
By Section~\ref{sec:trajectory}, $\dot Y_v=\omega_v\odot x$ for $v\ne s$, and by Appendix~\ref{app:trajectory-early}, $Y_v\le C_v\le\one$. The first inequality in \eqref{eq:drtangents}, applied at $u=Y_v$ with $\Delta=\omega_v\odot o=C_v-Y_v\ge0$, gives
\[
f(C_v)-f(Y_v)\le\ip{\nabla f(Y_v)}{\omega_v\odot o}.
\]
For $v\ne s$, the chain rule gives $\frac{d}{dv}f(Y_v)=\ip{\nabla f(Y_v)}{\omega_v\odot x}$. Adding this identity to the last display and rearranging gives
\begin{equation}\label{eq:ode}
\frac{d}{dv}f(Y_v)+f(Y_v)
\ge f(C_v)-\ip{\omega_v\odot\nabla f(Y_v)}{o-x}.
\end{equation}
Multiply \eqref{eq:ode} by $e^{v-1}$. The left side becomes $\frac{d}{dv}\bigl(e^{v-1}f(Y_v)\bigr)$. The function $v\mapsto e^{v-1}f(Y_v)$ is continuous on $[0,1]$ and continuously differentiable except at $v=s$, so integrating over $[0,1]$ gives $f(Y_1)-e^{-1}f(Y_0)=f(Y_1)-e^{-1}f(0)$, since $Y_0=0$. On the right side, $-\int_0^1e^{v-1}\ip{\omega_v\odot\nabla f(Y_v)}{o-x}\,dv=-\ip{q_s(f,x)}{o-x}$ by \eqref{eq:q}. It remains to integrate $e^{v-1}f(C_v)$. On $[0,s)$, Lemma~\ref{lem:early}, $\int_0^se^{v-1}dv=(e^s-1)/e$ and $\int_0^se^{v-1}(1-e^{-v})\,dv=(e^s-1-s)/e$ give
\[
\frac{e^s-1}{e}\bigl[f(o)-f(x\odot o)\bigr]
-\frac{e^s-1-s}{e}f(x\oplus o).
\]
On $[s,1]$, Lemma~\ref{lem:late}, $e^{v-1}e^{-v}=e^{-1}$, $\int_s^1dv=1-s$ and $\int_s^1(v-s)\,dv=(1-s)^2/2$ give
\[
\frac{(1-s)e^s}{e}f(o)-\frac{(1-s)(e^s-1)}{e}f(x\oplus o)
+\frac{(1-s)^2}{2e}\left(1+\frac{3s}{4}\right)f(x\oplus o).
\]
Adding the two contributions, the coefficient of $f(o)$ is $[e^s-1+(1-s)e^s]/e=\theta_s$, the coefficient of $f(x\odot o)$ is $-(e^s-1)/e=-\mu_s$, and the coefficient of $f(x\oplus o)$ is $-[e^s-1-s+(1-s)(e^s-1)]/e+\frac{(1-s)^2}{2e}\bigl(1+\frac{3s}{4}\bigr)=-\zeta_s$, as in \eqref{eq:trajectorycoeff}. Collecting all terms gives
\[
f(Y_1)\ge\theta_s f(o)-\mu_s f(x\odot o)-\zeta_s f(x\oplus o)
-\ip{q_s(f,x)}{o-x}+e^{-1}f(0).
\]
Since $f(0)\ge0$, this proves \eqref{eq:trajectorycompare}.

For \eqref{eq:trajectoryfree}, write $-\zeta_sf(x\oplus o)=-\ell f(x\oplus o)+(\ell-\zeta_s)f(x\oplus o)$. Multiplying the lattice inequality \eqref{eq:lattice} by $\ell\ge0$ gives $-\ell f(x\oplus o)\ge\ell f(x\odot o)-2\ell f(x)-\ell\ip{\nabla f(x)}{o-x}$. Substituting this into \eqref{eq:trajectorycompare} gives
\begin{equation}\label{eq:retainedunion}
\begin{split}
f(Y_1)\ge{}&\theta_s f(o)-(\mu_s-\ell)f(x\odot o)-2\ell f(x)\\
&+(\ell-\zeta_s)f(x\oplus o)-\ip{q_s(f,x)+\ell\nabla f(x)}{o-x}.
\end{split}
\end{equation}
If $\ell\ge\max\{0,\zeta_s\}$, then $(\ell-\zeta_s)f(x\oplus o)\ge0$, so this term can be dropped, and rearranging gives \eqref{eq:trajectoryfree}.
\end{proof}

\section{Proof of Lemma~\ref{lem:estimator}}\label{app:estimation}

\begin{proof}[Proof of Lemma~\ref{lem:estimator}]
The estimator \eqref{eq:onequery} makes one gradient query, at $Y_V(x)$, at $x$, or at $x\odot(\one-Z)$, with the query at $0$ in the box branch when $x=0$. Section~\ref{sec:trajectory} shows that $Y_V(x)\in[0,x]$, and $x\odot(\one-Z)\in[0,x]$ because $Z$ lies in the cube. Hence the query lies in $[0,x]\subseteq\K$. We first compute the conditional mean of each branch and then bound the second moments.

We begin with the trajectory branch. Let $c_0=1-e^{-1}=\int_0^1e^{v-1}\,dv$. If $Q$ is uniform on $[0,1]$, then $V=\log(1+(e-1)Q)$ satisfies $\Pr[V\le v]=\Pr[Q\le(e^v-1)/(e-1)]=(e^v-1)/(e-1)$ for $v\in[0,1]$, so $V$ has density $e^v/(e-1)=e^{v-1}/c_0$. Hence, by \eqref{eq:q},
\begin{equation}\label{eq:trajmean}
\E\bigl[c_0\omega_V(x)\odot\nabla f(Y_V(x))\mid f,x\bigr]=\int_0^1e^{v-1}\omega_v(x)\odot\nabla f(Y_v(x))\,dv=q_s(f,x).
\end{equation}

Next comes the box branch. If $x=0$, then $\widehat h=0=h(f,x)$, so assume $x\ne0$. Let $\sigma=(\sigma_1,\ldots,\sigma_d)$ be the permutation that sorts the weights $w_i=x_i(1-x_i)$ of $a=\one-x$, as in Lemma~\ref{lem:chain}. Fix a coordinate $i$ with $x_i>0$, let $j$ be its position, so that $\sigma_j=i$, and condition on $I=i$. Write $Z$ for the sampled chain point and $z=x\odot(\one-Z)$ for the query point. Let $\underline A_i$ and $\overline A_i$ be the averages in \eqref{eq:chainfield} for $F=J_x$ and $a=\one-x$, with the chain points $L^{j-1}$ and $U^{j-1}$ in the sorted order. By the chain rule, $\partial_iJ_x(Z)=-x_i[\nabla f(z)]_i$. On the lower segment, $Z=L^{j-1}+Sr_ie_i$ with $S$ uniform on $[0,1]$, so
\[
\E[\partial_iJ_x(Z)\mid I=i,\text{lower}]=\int_0^1\partial_iJ_x(L^{j-1}+vr_ie_i)\,dv=\underline A_i.
\]
In the same way, $\E[\partial_iJ_x(Z)\mid I=i,\text{upper}]=\overline A_i$. The lower segment has probability $1-r_i$, so
\[
\E\bigl[-x_i[\nabla f(z)]_i\mid I=i\bigr]=(1-r_i)\underline A_i+r_i\overline A_i=g_i(J_x,\one-x).
\]
Given $I=i$, the estimate $\widehat h$ has the single nonzero coordinate $i$, equal to $(x_i/\rho_i)[\nabla f(z)]_i$. Coordinate $i$ is drawn with probability $\rho_i$, so its expectation is $\rho_i\cdot(1/\rho_i)\cdot x_i\,\E\bigl[[\nabla f(z)]_i\mid I=i\bigr]=-g_i(J_x,\one-x)$, which is $h_i(f,x)$ by \eqref{eq:boxdefs}. A coordinate with $x_i=0$ is never drawn, and $h_i(f,x)=0$ because each average in \eqref{eq:chainfield} for $J_x$ carries the factor $x_i$. Hence
\begin{equation}\label{eq:boxmean}
\E[\widehat h\mid f,x]=h(f,x).
\end{equation}

Branch $\nu$ is drawn with probability $\pi_\nu$, and its output is divided by $\pi_\nu$. Averaging over $\nu$ therefore gives
\[
\E[\widehat G\mid f,x]=\lambda\,\E\bigl[c_0\omega_V(x)\odot\nabla f(Y_V(x))\mid f,x\bigr]+\lambda\ell\nabla f(x)+\lambda b_s\E[\widehat h\mid f,x].
\]
By \eqref{eq:trajmean}, \eqref{eq:boxmean} and \eqref{eq:combinedfield}, the right side equals $\lambda\bigl[q_s(f,x)+\ell\nabla f(x)+b_sh(f,x)\bigr]=G(f,x)$. This is the first identity in \eqref{eq:Gmoment}.

We now turn to second moments. Since $0\le\omega_V\le\one$ and $Y_V(x)\in\K$, $\norm{\omega_V(x)\odot\nabla f(Y_V(x))}_2\le\norm{\nabla f(Y_V(x))}_2\le B$. Since $x\in\K$, $\norm{\nabla f(x)}_2\le B$. For the box branch with $x\ne0$, let $m_i=\max\{|\partial_if(0)|,|\partial_if(x)|\}$. Since $0\le z\le x$ and $\nabla f$ is nonincreasing, $\partial_if(x)\le\partial_if(z)\le\partial_if(0)$, so $|\partial_if(z)|\le m_i$. Given $I=i$, the returned vector has squared norm $(x_i/\rho_i)^2[\nabla f(z)]_i^2\le(x_i/\rho_i)^2m_i^2$. Averaging over $I$ gives
\begin{equation}\label{eq:boxmoment}
\E\norm{\widehat h}_2^2
\le\sum_{i:x_i>0}\rho_i\frac{x_i^2}{\rho_i^2}m_i^2
=\norm{x}_2^2\sum_{i:x_i>0}m_i^2
\le\norm{x}_2^2\bigl(\norm{\nabla f(0)}_2^2+\norm{\nabla f(x)}_2^2\bigr)
\le2B^2R^2.
\end{equation}
The equality uses $\rho_i=x_i^2/\norm{x}_2^2$, the second inequality uses $m_i^2\le\partial_if(0)^2+\partial_if(x)^2$, and the third uses $0,x\in\K$ and $\norm{x}_2\le R$. If $x=0$, then $\widehat h=0$ and \eqref{eq:boxmoment} holds trivially. The quantities $m_i$ are used only in this bound, and the algorithm does not evaluate them.

Conditioning on the branch and using the three bounds,
\[
\begin{aligned}
\E[\norm{\widehat G}_2^2\mid f,x]
&=\lambda^2\Bigl[\frac{c_0^2}{\pi_0}\E\norm{\omega_V(x)\odot\nabla f(Y_V(x))}_2^2+\frac{\ell^2}{\pi_1}\norm{\nabla f(x)}_2^2+\frac{b_s^2}{\pi_2}\E\norm{\widehat h}_2^2\Bigr]\\
&\le\lambda^2B^2\Bigl[\frac{c_0^2}{\pi_0}+\frac{\ell^2}{\pi_1}+\frac{2R^2b_s^2}{\pi_2}\Bigr].
\end{aligned}
\]
For the probabilities in \eqref{eq:branchprob}, the bracket equals $\Sigma_R(c_0+\ell+\sqrt2\,b_sR)=\Sigma_R^2$, and $\lambda \Sigma_R=\Gamma_R$. This gives the second claim in \eqref{eq:Gmoment}. Finally, for any positive probabilities, the Cauchy--Schwarz inequality gives
\[
(c_0+\ell+\sqrt2\,b_sR)^2\le\Bigl(\frac{c_0^2}{\pi_0}+\frac{\ell^2}{\pi_1}+\frac{2R^2b_s^2}{\pi_2}\Bigr)(\pi_0+\pi_1+\pi_2),
\]
so the probabilities in \eqref{eq:branchprob} minimize the bracket.
\end{proof}

\FloatBarrier
\bibliography{references}
\end{document}

%% file: math_commands.tex
\usepackage{amsmath,amsfonts,bm}

\def\eqref#1{(\ref{#1})}

\def\1{\bm{1}}

\DeclareMathAlphabet{\mathsfit}{\encodingdefault}{\sfdefault}{m}{sl}
\SetMathAlphabet{\mathsfit}{bold}{\encodingdefault}{\sfdefault}{bx}{n}

\newcommand{\E}{\mathbb{E}}

\newcommand{\R}{\mathbb{R}}

